\documentclass{article}

\usepackage[final]{corl_2022}
\usepackage{enumitem}

\usepackage{amssymb}  % assumes amsmath package installed
\usepackage{xurl}
\usepackage{setspace}
\usepackage[skip=0pt]{caption}
\usepackage{graphicx} % for pdf, bitmapped graphics files
\usepackage{epsfig} % for postscript graphics files
\usepackage{mathptmx} % assumes new font selection scheme installed
\usepackage{amsmath} % assumes amsmath package installed
\usepackage{float}
\usepackage{bm}
\usepackage{mathtools}
\usepackage[linesnumbered]{algorithm2e}
\usepackage{mwe}
\usepackage{float}
\usepackage{ulem}
\usepackage{algorithmicx}
\usepackage{multirow}
\usepackage[numbers]{natbib}
\usepackage{algpseudocode}
\usepackage{lipsum}
\usepackage{amsthm}
\usepackage[linewidth=1pt]{mdframed}

\usepackage{comment}

\newtheorem{assumption}{Assumption}
\newtheorem{theorem}{Theorem}
\newtheorem*{theorem-non}{Theorem}
\newtheorem{lemma}[theorem]{Lemma}
\newtheorem{definition}{Definition}
\newtheorem{proposition}{Proposition}
\newtheorem*{proposition-non}{Proposition}

\title{
Data Games: A Game-Theoretic Approach to Swarm Robotic Data Collection
}

\author{Oguzhan Akcin$^{1}$, Po-han Li$^{1}$, Shubhankar Agarwal$^{1}$ and Sandeep P. Chinchali$^1$ % <-this % stops a space
\thanks{$^{1}$ Department of Electrical and Computer Engineering (ECE), The University of Texas at Austin, Austin, TX
{\tt\small \{oguzhanakcin,pohanli\}@utexas.edu}, {\tt\small \{somi.agarwal,sandeepc\}@utexas.edu}}}%

\begin{document}

%\doublespacing
%\setstretch{3.0}

\newcommand{\somi}[1] {{\color{red} \textbf{[Somi]: #1}}}
\newcommand{\pohan}[1] {{\color{blue} \textbf{[Po-han]: #1}}}
\newcommand{\rebuttal}[1] {{\color{black} #1}}

\newcommand{\xit}{x_i^t}
\newcommand{\yhatit}{\hat{y}_i^t}
\newcommand{\yit}{y_i^t}
\newcommand{\action}[2]{a_{#2}^{#1}}
\newcommand{\networkparam}[1]{\theta^{#1}}
\newcommand{\numdataforclass}{N_y_j}
\newcommand{\clouddatasetr}{\mathcal{D}_c^r}
\newcommand{\devicedataset}[2]{\mathcal{D}_{#2}^{#1}}
\newcommand{\devicedatasetri}{\rho_{\mathcal{D}_i^r}}
\newcommand{\yhatclass}{\hat{y}_{N_{\text{class}}}}
\newcommand{\yclass}{y_{N_{\text{class}}}}
\newcommand{\datasettarget}{\rho_{\mathcal{D}_{\text{target}}}  }
\newcommand{\condhy}[2]{p_i^r(\hat{y}_{#1} | y_{#2} )}
\newcommand{\condhyni}[2]{p^r(\hat{y}_{#1} | y_{#2} )}
\newcommand{\condyh}[2]{p_i^r( y_{#1} | \hat{y}_{#2} )}
\newcommand{\clouddataset}[1]{\rho_{\mathcal{D}_{c}^{#1}}}
\newcommand{\kl}[1]{\mathcal{L}(\clouddataset{#1} , \datasettarget)}
\newcommand{\ndevice}{N_{\text{robot}}}
\newcommand{\nclass}{N_{\text{class}}}
\newcommand{\ncache}{N_{\text{cache}}}
\newcommand{\reals}{\mathbb{R}}
\newcommand{\realsplus}{\mathbb{R}_{+}}

\newcommand{\Greedy}{\textsc{Greedy}}
\newcommand{\Oracle}{\textsc{Oracle}}
\newcommand{\Uniform}{\textsc{Uniform}}
\newcommand{\Interactive}{\textsc{Interactive}}
\newcommand{\Lowerbound}{\textsc{Lower-Bound}}

% for appendix
\newcommand{\dataobmatrix}[2]{P_{#1}^{#2}} % 1:i, 2:r
\newcommand{\fspace}[2]{H_{#1}^{#2}}
\newcommand{\faction}[2]{{v}_{#1}^{#2}}

% OLD
%\newcommand{\loss}[2]{{L}_{#1}^{#2}}
\newcommand{\loss}[2]{{\mathcal{L}}_{#1}^{#2}}

\def\NoNumber#1{{\def\alglinenumber##1{}\State #1}\addtocounter{ALG@line}{-1}}

\maketitle

%%%%%%%%%%%%%%%%%%%%%%%%%%%%%%%%%%%%%%%%%%%%%%%%%%%%%%%%%%%%%%%%%%%%%%%%%%%%%%%%
\begin{abstract}
    Fleets of networked autonomous vehicles (AVs) collect terabytes of sensory data, which is often transmitted to central servers (the ``cloud'') for training machine learning (ML) models. Ideally, these fleets should upload all their data, especially from rare operating contexts, in order to train robust ML models. However, this is infeasible due to prohibitive network bandwidth and data labeling costs. Instead, we propose a cooperative
data sampling strategy where geo-distributed AVs collaborate to collect a diverse ML training dataset in the cloud.
%Our key insight is that each AV observes rare events specific to its unique operating context and, therefore, should cooperate with others to decide what specialized data to upload. 
Since the AVs have a shared objective but minimal information about each other's local data distribution and perception model, we can naturally cast cooperative data collection as an $N$-player mathematical game.
We show that our cooperative sampling strategy uses minimal information to converge to a centralized oracle policy with complete information about all AVs. Moreover, we 
theoretically characterize the performance benefits of our game-theoretic strategy compared to greedy sampling.
Finally, we experimentally demonstrate that our method outperforms standard benchmarks by up to $21.9\%$ on 4 perception datasets, including for autonomous driving in adverse weather conditions.  Crucially, our experimental results on real-world datasets closely align with our theoretical guarantees.

%However, each AV has minimal information about each other's local data distribution, perception model uncertainty, and unique operating contexts. Since the AVs have a shared objective but minimal information about each other, we can naturally cast cooperative data collection as an $N$-player convex game. 

\end{abstract}

%%%%%%%%%%%%%%%%%%%%%%%%%%%%%%%%%%%%%%%%%%%%%%%%%%%%%%%%%%%%%%%%%%%%%%%%%%%%%%%%

\section{Introduction}

Envision a fleet of autonomous vehicles (AVs) that observes heterogeneous street scenery, weather conditions, and rural/urban traffic patterns. 
To train robust ML models for perception or trajectory prediction, these AVs should share as much diverse fleet data as possible in the cloud, while balancing network bandwidth, data storage, and labeling costs.\footnote{A single AV can measure over 20-30 Gigabytes (GB) per second of video and LiDAR data \cite{intel} while a typical 5G wireless network only provides 10 Gbps of bandwidth for \textit{multiple} users \cite{Qualcomm5G}.} 
Given these constraints, we argue that AVs must \textit{coordinate} how to sample rare, out-of-distribution (OoD) data with common examples based on their unique local data distributions. For example, if only a few AVs operate in heavy snow, they should specialize in sending snowy images to the cloud, while others should send data from more common scenarios like sunny weather. Since the AVs have a shared target data distribution (objective) but limited information on each other's
local data distribution and potentially private ML models, our key contribution is to cast data collection as a \textbf{N-Player mathematical game}.

In our game-theoretic formulation (Fig. \ref{fig:expr_arch}), the AVs exchange minimal information to choose a data sampling strategy (what limited subset of data-points to upload). Importantly, we prove that an AV fleet will quickly converge to a \textbf{Nash equilibrium} (i.e., a fixed point where each robot does not change its sampling strategy) \cite{nash1950equilibrium, nash1951non} with bounded communication. Morever, our practical formulation accounts for perceptual uncertainty from \textit{imperfect} computer vision models and heterogenous
local data distributions.  As such, to the best of our knowledge, we are the first to cast data sampling from networked robots as a mathematical game. In summary, our key contributions are:

\begin{enumerate}[leftmargin=*,noitemsep,topsep=0pt]
    \item We provide a novel formulation for distributed data collection as a \textit{potential} game \cite{Rosen1965Existence} since the robots attempt to minimize a common convex objective function that incentivizes them to reach a balanced target data distribution in the cloud. 
        %Section \ref{app:whypotential} provides more details on potential games.
    We prove that our strategy converges to a centralized oracle policy and, under mild assumptions, converges in a single iteration. 
    \item We provide theoretical performance bounds characterizing the benefits of our game-theoretic approach compared to greedy, individual behavior. 
    \item We show that our proposed strategy outperforms competing benchmarks by $21.9\%$ on 4 datasets, including the challenging Berkeley DeepDrive autonomous driving dataset \cite{yu2020bdd100k}.
\end{enumerate}
\begin{figure}[t]
    \includegraphics[width=1.0\columnwidth]{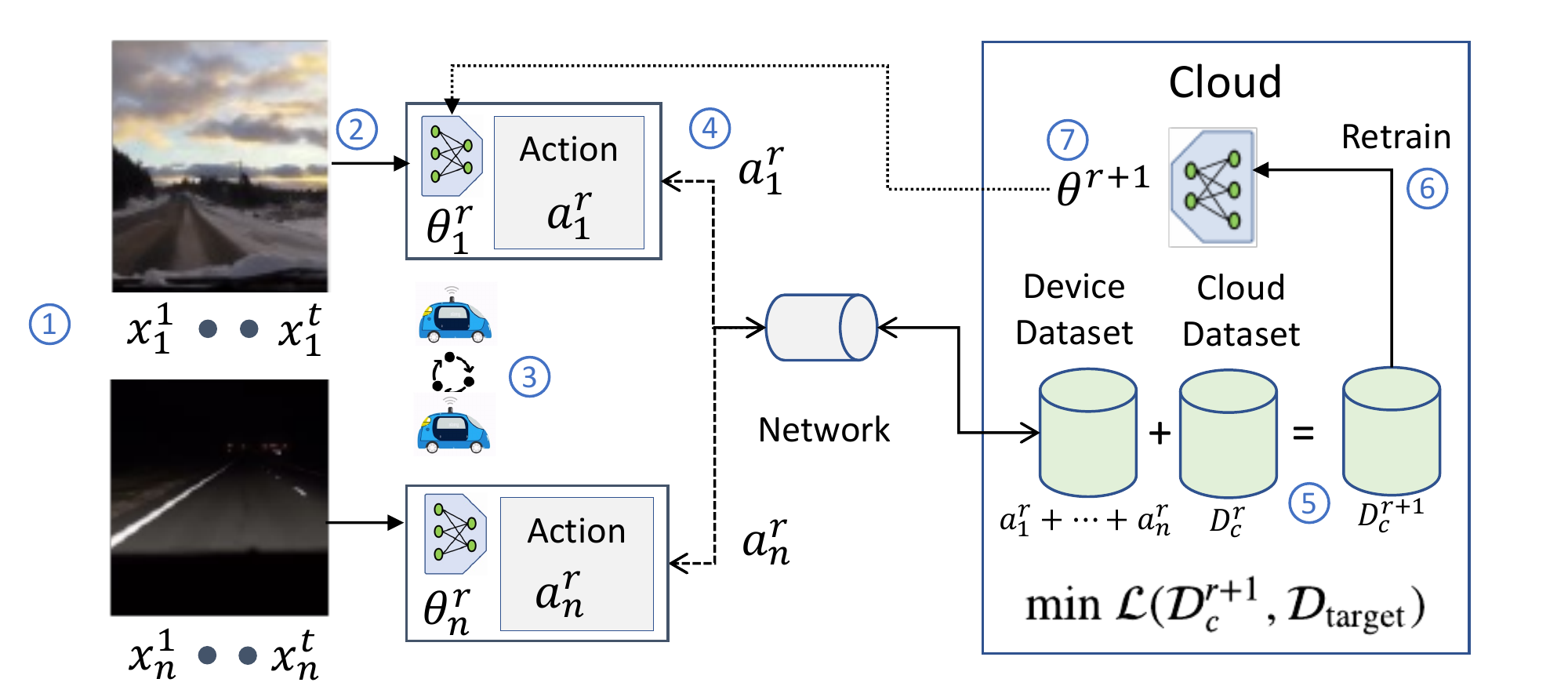}
\caption{\small{\textbf{Game-Theoretic Data Collection:} 
    Each step in our cooperative algorithm is numbered in blue. First, each AV $i$ observes a sequence of images $x^t_i$ in each round $r$ of data collection (step 1). Then, it classifies each image $x^t_i$ with a local vision model with parameters $\networkparam{r}_i$ (step 2). Then, it samples a limited set of $N_{\text{cache}}$ images according to its action policy $\action{r}{i}$, which governs what distribution of data-points to upload. Crucially, the action
    $\action{r}{i}$ is chosen cooperatively with other AVs using a distributed optimization problem (step 3). Next, each AV transmits its local cache of data-points to the cloud (step 4). The current cloud dataset, $\mathcal{D}_{c}^r$, is updated with the new uploaded data-points $\action{r}{i}$ (step 5). The combined cloud dataset, $\mathcal{D}_c^{r+1}$, can be used to periodically re-train new model parameters $\networkparam{r+1}$
    (step 6), which are then downloaded by the AVs (step 7). All AVs share a goal of
    minimizing the distance between the collected cloud dataset $\mathcal{D}_c^{r+1}$ (green) and the target $\mathcal{D}_{\text{target}}$. 
    \label{fig:expr_arch}
} }
    \vspace{-1em}
\end{figure}

\textbf{Related Work: }
Data collection from networked robots is related to cloud robotics \cite{goldbergwebsite, kehoe2015survey, kuffner2010cloud, Chinchali2010HarvestNetMV, goldberg2013cloud, chinchali2019RSS, tanwani2019fog, geng2021decentralized, pujol2021fog} and  active learning \cite{Cohn96activelearning, pmlr-v70-gal17a, settles2009active, tong2001active, ActLearn4wireless}. In such prior works, robots either send \textit{all} their data to the cloud
or select samples individually without coordination. 
In contrast, we exploit the fact that networked AVs can coordinate how to sample rare data to achieve a better outcome (i.e., balanced data distribution).

Federated learning (FL) \cite{fl, fedler, XIANJIA2021135, OptSampling4FL, armacki2022personalized, WirelessDSGD, PrivacyAsyncFL, FLSurveyPaper, caldas2018expanding} enables a fleet of mobile devices to train ML models on local private data and only share anonymized gradient updates with the cloud. However, our work is fundamentally different, and even complementary, to standard FL.  First, FL makes the restrictive assumption that each robot has perfectly labeled local data,
which is infeasible for AVs that observe rare, OoD images. Instead, we address a practical scenario where robots run local inference with only an \textit{imperfect} vision model that guides data collection.
Moreover, FL does not statistically sample data but trains on all of it locally, while our approach only uploads a limited set of images to reduce network and data labeling costs. Finally, we assume robots only receive ground-truth labels for the uploaded data in the cloud, which is required for training on rare classes. 

Our setting is a non-cooperative game since the robots do not explicitly form coalitions and act with minimal information about each other \cite{Rosen1965Existence, KIM20021219, wang1988theory, driessen2013cooperative, Stuart2001, tijs2003introduction}.
Specifically, our setting is a potential game since each robot attempts to maximize a shared concave objective function (the common \textit{potential} function) that rewards progress towards a balanced target data distribution in the cloud.
%Specifically, our setting is a potential game with players maximizing the potential function, a concave reward, since each robot attempts to minimize a shared convex objective function that rewards progress towards a balanced target data distribution in the
%cloud. 
As detailed in Sec. \ref{sec:formulation} and Appendix \ref{app:whypotential}, changes in the common potential function directly translate to changes in each robot's policy towards a Nash Equilibrium. While concave games have been applied to problems such as wireless network resource
allocation \cite{han2012game}, ours is the first work to contribute a game-theoretic formulation for distributed data collection from \rebuttal{a fleet of robots}.

\section{Problem Formulation}
\label{sec:formulation}
\begin{figure}[t]
    \begin{minipage}[c]{0.35\textwidth}
        \includegraphics[width=\columnwidth]{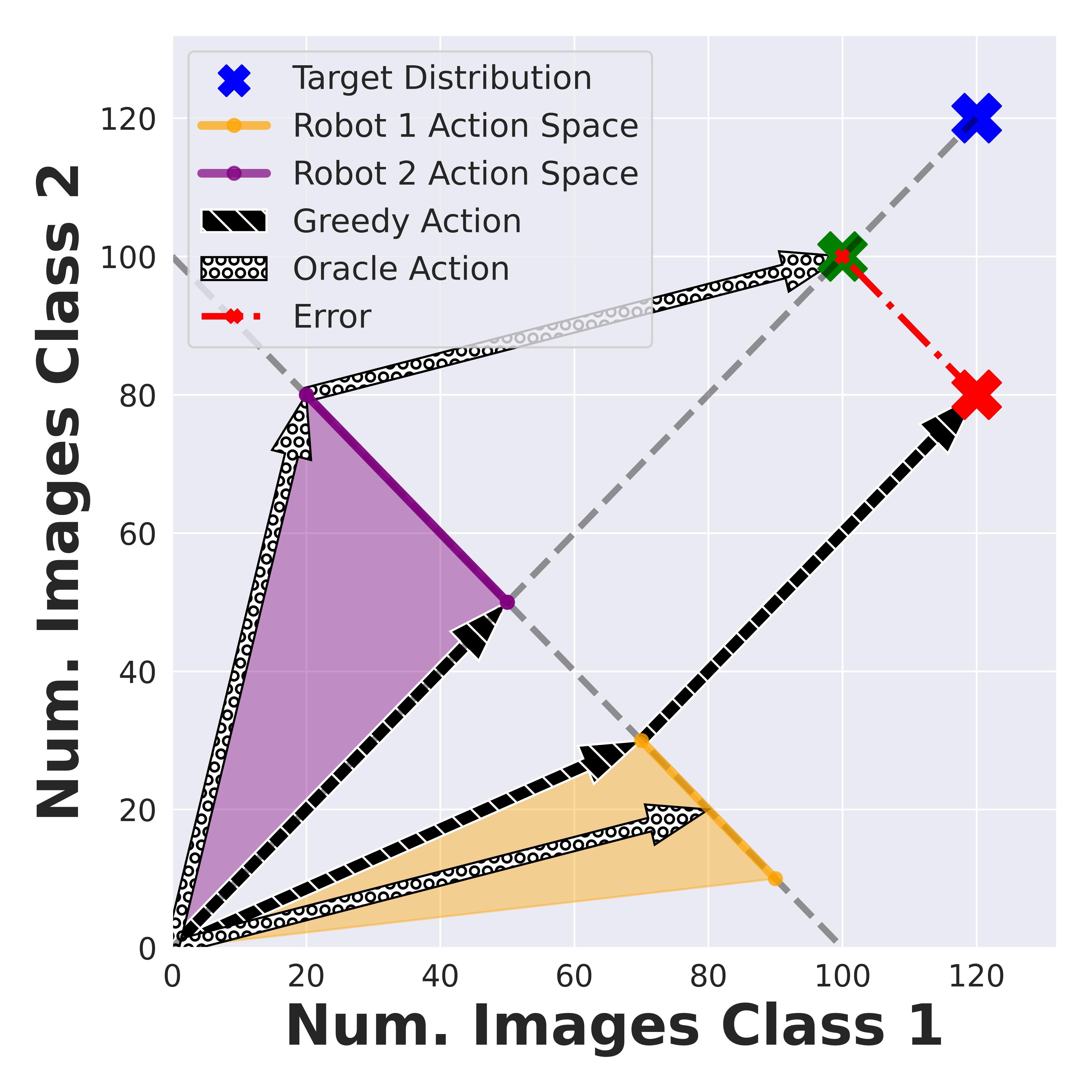}
      \end{minipage}\hfill
      \begin{minipage}[c]{0.65\textwidth}
        \caption{\small{\textbf{Why Cooperate?}
          Consider a toy example with only 2 classes and 2 robots. The axes represent the number of data-points for each class. Our goal is to reach the target distribution (blue cross) where each class has 120 data-points, represented by $(120, 120)$. 
          The robots start at $(0,0)$ with no data-points in the cloud. The possible combinations that can be uploaded from robots 1 and 2 are shown as the shaded feasible action spaces (yellow and purple). This shaded region is 
          determined by the robot's local data distribution and vision model accuracy (Def. \ref{def:fspace}). \Greedy~ (black) individually calculates the projection of the target distribution onto each robot's feasible action space, but the sum of actions may not be optimal, leading to a high error (red). However, \Oracle~ accounts for the two robots' action spaces and thus minimizes the error between the target dataset and the sum of actions (grey). 
%between robots reduces our distance to the desired target distribution. 
        \label{fig:toy_ex}
    } }
    \end{minipage} 
\vspace{-1.5em}
\end{figure}

We now formulate a practical scenario, shown in Fig. \ref{fig:expr_arch}, where distributed robots collect data to train a robust ML model in the cloud. Our goal is to select an appropriate \textit{action} for each robot, specifically the data-points it should upload, so that the overall collected cloud dataset closely matches a given target, such as an equal distribution over all classes. Fig. \ref{fig:toy_ex} intuitively depicts data sampling. Our general formulation applies to any robot constrained by network, storage, or labeling costs,
ranging from Mars Rovers constrained by the Deep Space Network ($<5$ Mbps) \cite{DeepSpaceNetwork} or AV fleets.
%AVs that each measure terabytes of data in a few hours \cite{intel}. 

%. Specifically, each AV $i$ can potentially observe a different data distribution of image, LiDAR, and radar datapoints $\xit$ across diverse locations. Each AV must intelligently sample data to upload at the end of each day $r$ due to network, storage, or data labeling costs. 

%\paragraph*{Robot Observations and Perception Model: }
%We consider a robotic fleet with $\ndevice$ robots, where each robot $i$ observes a data-point $\xit$ at discrete time $t$ from its local environment (i.e. camera image or LiDAR scan). Each robot processes sensory observation $\xit$ using a perception model, denoted by $\hat{y} = f(x; \networkparam{r})$, where $f(\cdot; \networkparam{r})$ is a model with parameters $\networkparam{r}$ at round $r$, $\hat{y}$ is predicted label by the model for the collected data-point $x$, and $y$ is corresponding ground-truth label. Additionally, we also have a confusion matrix, $C_r$ (Eq. \ref{eq:confusion_mat}), for the classification model $f(\cdot; \networkparam{r})$ (discussed in Appendix \ref{sec:confusion_matrix}). 

%\subsection{Classification Task}
%\paragraph*{Robot Perception Model: }
\textbf{Robot Perception Model: }
For a simple exposition, we first consider a general computer vision classification task with $N_{\text{class}}$ classes. The dataset used for training the model is stored in the cloud. 
Each period of data collection, such as a day, is denoted by a \textit{round} $r$ and data is uploaded to the cloud at the end of a round $r$.
The cumulative dataset stored in the cloud at the end of round $r$ is denoted by $\clouddatasetr$, whose size is given by $N_{\clouddatasetr} = |\clouddatasetr|$. 
$N_{y_j}$ denotes the number of class $j$ data-points in the dataset $\clouddatasetr$. Therefore, the distribution of classes in the dataset $\clouddatasetr$ is denoted by $\rho_{\clouddatasetr} = \left[ N_{y_0}, N_{y_1}, \cdots, N_{y_{N_{\text{class}}}}  \right]$. 
Each robot $i$ has a perception model, such as a deep neural network (DNN), where local inference is denoted by $\hat{y} = f(x; \networkparam{r}_i)$. Here, $f(\cdot; \networkparam{r}_i)$ is a model with parameters $\networkparam{r}_i$ at round $r$, $\hat{y}$ is the predicted label for input $x$, and $y$ is the corresponding ground-truth label. 

Importantly, the models can be \textit{imperfect} -- each model has a confusion matrix, $C^r_i \in \reals^{\nclass \times \nclass}
$ (Eq.
\ref{eq:confusion_mat}) that captures the probability of predicting class $\hat{y}_j$ for an image with true class $y_j$, denoted by $p(\hat{y}_j | y_j)$. 
In practice, one of the $\nclass$ classes can represent an ``unknown'' category while the rest of the $\nclass-1$ classes can represent a mixture of rare and well-understood concepts. 
Further details on the confusion matrix are provided in Appendix \ref{sec:confusion_matrix}.
Finally, while we use a (likely imperfect) classification model to sample images, the uploaded data can be used to train models for diverse tasks such as object detection, semantic segmentation etc.

% It is important to note that usually, the confusion matrix is constructed from $p(\hat{y}_j \cap y_j )$, but in this work we construct the confusion matrix from $p(\hat{y}_j | y_j )$.

%\paragraph*{Robot Fleet: }
\textbf{Robot Fleet: }
We consider a fleet of $\ndevice$ robots, where each robot $i$ collects a data-point $\xit$ at time $t$ from its local environment (i.e., camera image or LiDAR scan). 
The distribution of true classes observed by a robot $i$ in round $r$ is denoted by $p^r_i(y) \in \realsplus^{\nclass}$, which sums to one over the $\nclass$ classes. 
From this distribution, a robot $i$ observes a large dataset of images
on round $r$ denoted by $\devicedataset{r}{i}$ of size $|\devicedataset{r}{i}| = N_{i}^r$.
However, to limit network bandwidth and data labeling costs, each robot $i$ can only upload
$\ncache \ll N^r_i$ images to the cloud at the end of round $r$, which it stores in an on-board cache within the round. The size of $N_{\text{cache}}$ can be flexibly set by a roboticist based on data upload and labeling budgets.
The class predictions, $\hat{y}_j$, are generated by running local inference of the classification model, $\hat{y}_i = f(x_i^t; \networkparam{r}_i)$, for the collected data-points $\xit$. Finally, $p^r_i(\hat{y})$ denotes the distribution of \textit{predicted} classes observed by robot $i$.

\begin{assumption} \label{assum:large_data_collected}
The number of data-points collected by a robot on any round $r$, $N_{i}^r$, is significantly greater than the size of the local robot cache, $N_{i}^r \gg N_{\text{cache}}$.
\end{assumption} 

This is a valid assumption since each robot will collect much more data compared to the amount it can economically upload. Our formulation is extremely general -- each robot can have different (or the same) model parameters $\networkparam{r}_i$ and observe a different distribution $p^r_i(y)$ of the $N_{\text{class}}$ classes. 
%\somi{For the ease of notation, we drop $i$ from model parameters $\networkparam{r}{i}$ since robots often have uniform perception models but
%quite different local data distributions.}

\paragraph*{Robot Statistical Sampling Action: }

At each round $r$, each robot $i$ takes an action which determines how many data-points of each class to send to the cloud. We define each robot $i$'s action at round $r$ as $\action{r}{i} = \left[ N_{y_0}, N_{y_1}, \cdots, N_{y_{N_{\text{class}}}}  \right]$, i.e. the number of data-points of each class $j$. Our key technical innovation will be to illustrate \textit{how} to cooperatively select an optimal action. Importantly, since each robot $i$ has an
\textit{imperfect} perception model with confusion matrix $C^r_i$, there is uncertainty in the effect of taking any action $\action{r}{i}$. As such, our natural next step is to define the set of feasible actions any robot can upload given its local data distribution and perceptual uncertainty.

\begin{definition}[Feasible data matrices]
A feasible data matrix, $\dataobmatrix{i}{r} \in \mathbb{R}^{\nclass\times\nclass}$, of robot $i$ in round $r$ is the probability matrix defined as: 
$$\dataobmatrix{i}{r}=[p^{r}_{i,1},...,p^{r}_{i,\nclass}],$$
where $p^r_{i,j}=\frac{C_{i,j}^r*p_i^r(y)}{\|C_{i,j}^r*p_i^r(y)\|_1}=p(y|\hat{y}_j)\in\mathbb{R}^{\nclass}, \forall j=1,...,\nclass$. We use $*$ as element-wise multiplication of vectors, $\|\cdot\|_1$ as the $L_1$ norm, and the second subscript $j$ to denote the $j$-th column of a matrix. We assume $\dataobmatrix{i}{r}$ has linearly independent columns, so there exists a left inverse.
    In other words, we assume the mapping from action to feasible action (Defs. \ref{def:fspace}, \ref{def:faction}) is one-to-one. This assumption is justified in the Appendix due to space limits.
\end{definition}

\begin{definition}[Feasible spaces of robots]

A feasible space, $\fspace{i}{r}$, of robot $i$ in round $r$ is the set of feasible data-points the robot can send to the cloud: 
$$\fspace{i}{r}=\{\faction{i}{r}=\dataobmatrix{i}{r} \action{r}{i} ~|~ \mathbf{1}^\top\action{r}{i}\leq \ncache, \action{r}{i} \in \mathbb{R}_{+}^{\nclass}\}.$$ 
$\fspace{i}{r}$ is the convex hull of all columns of $\dataobmatrix{i}{r}$ and $\mathbf{0}$. 
    To simplify notation, $\faction{i}{r} = \dataobmatrix{i}{r} \action{r}{i}$ represents a feasible action $\faction{i}{r}$, which is obtained by multiplying  an intended action $\action{r}{i}$ by the feasible data matrix $\dataobmatrix{i}{r}$. 
\label{def:fspace}
\end{definition}

Intuitively, the feasible space (see Fig. \ref{fig:toy_ex}) represents the expected number of datapoints uploaded per class but not the exact number due to perceptual uncertainty. 
%Fig. \ref{fig:toy_ex} illustrates the feasible space $\fspace{i}{r}$ in a toy example.
Each robot uploads $\ncache$ data-points sampled from action $\action{r}{i}$, which is pooled in the cloud. 
%After each robot takes an action and uploads data to the cloud, we can re-train a new perception model on the new dataset $\clouddataset{r}$.
%After each robot takes an action and uploads data to the cloud, 
%we can re-train a new perception model on the new dataset $\clouddataset{r}$.
We assume we only get ground-truth labels $y$ in the cloud, 
since the limited cache of images can be scalably annotated by a human. 
Then, we re-train a new perception model on the new dataset $\mathcal{D}_c^{r+1}$.
Each robot then downloads the new model parameters $\networkparam{r+1}_i$, along with the new confusion matrix and latest cloud dataset distribution, $\clouddataset{r+1}$.
Our formulation is general -- models and confusion matrices do not have to be updated every round $r$ and we can, for example, simply re-train a model after $M$ rounds of data collection.

\paragraph*{Collective Goal: Achieving a Target Data Distribution} 

Often, we want to achieve a balanced dataset in the cloud with ample representation of rare events in order to train a robust ML model. As such, the shared goal of all the robots is to achieve \textit{any} user-specified target dataset $\datasettarget$,  which defines the number of data-points of each class the robots want to collect in the cloud. The fleet's goal is to choose actions $\action{r}{i}$; $\forall \hspace{0.1cm }i=1, \ldots ,\ndevice$ at round $r$ to
\textit{collectively} reduce a strictly convex distance
metric, denoted by $\kl{r}$, penalizing the difference between the current cloud dataset $\clouddataset{r}$ and the target dataset $\datasettarget$. Our general framework can handle any strictly convex distance metric, such as the $L_2$ norm or the Kullback-Leibler (KL) Divergence \cite{kullback1951information}. Since all robots have a common goal to maximize the negative loss $-\kl{r}$, which is a concave potential function, our setting is a potential game with concave rewards
(see Sec. \ref{app:whypotential}). 

%A potential in a game is defined in chapter 8 of \cite{tijs2003introduction}. 
%It is a function indicating the incentives of all players (in our case robots), and any game with a potential is called a potential game. 
%Typically, the goal of a player is to maximize its incentive expressed with the potential. In our case, minimizing the loss function $\kl{r+1}$ in Eq. \ref{eq:dataset_inter} is the common goal for all robots, so the potential is the negative of the loss function,  $-\kl{r+1}$. 
%Note that the potential $-\kl{r+1}$ is concave since the loss function $\kl{r+1}$ is convex.
%\label{app:whypotential}

%That is, all robots need to choose an action such that they collectively take the cloud dataset $\clouddataset{r}$ close to $\datasettarget$ in terms of a convex distance metric (i.e. L2-norm or KL-divergence).

%Our general problem formulation applies to a wide variety of robots, ranging from Mars Rovers constrained by the Deep Space Network ($<5$ Mbps of bandwidth) or AVs that each measure terabytes of data in a few hours. Specifically, each AV $i$ can potentially observe a different data distribution of image, LiDAR, and radar datapoints $\xit$ across diverse locations. Each AV must intelligently sample data to upload at the end of each day $r$ due to network, storage, or data labeling costs. 

\paragraph*{Centralized Oracle Action Policy: }
We now provide a formal optimization problem for distributed data collection. To provide key insight, we first describe a centralized ``oracle'' solution that has perfect information about all robots $i$, namely their confusion matrix $C^r_i$ and statistics of their data distribution $p^r_i(y)$. Then, we formalize a greedy, individualized approach and our interactive game-theoretic approach that matches the oracle policy's performance. 

An oracle action policy, denoted by \Oracle, has access to all robots' data distributions and confusion matrices $C^r_i$. The oracle calculates each robot's action $\action{r}{i}$ by solving the convex optimization problem in Eq. \ref{eq:oracleopt}. The constraint (Eq. \ref{eq:distrib_action_oracle}) ensures that the actions $\action{r}{i}$ do not exceed
the cache limit $\ncache$. Eq. \ref{eq:dataset_oracle} shows the update of the cloud dataset for round $r+1$ based on the actions $\action{r}{i}$ taken in the feasible space, $\dataobmatrix{i}{r} \action{r}{i}$, by each robot for round $r$, which we now detail.

A key subtlety is to update the cloud dataset $\clouddataset{r+1}$ 
by merging the current cloud dataset $\clouddataset{r}$ and each robots' uploaded dataset $\action{r}{i}$. However, each robot's action is imperfect -- it might think it is uploading class $j$ but due to perceptual uncertainty it might actually upload another class $j'$. 
Specifically, the robot's transmitted dataset $\action{r}{i}$ is calculated from the predicted class labels $\hat{y}_j$ and not the true class labels $y_j$, which are not available on-robot. However, we can use predicted class probabilities $p_i^r(\hat{y}_j)$ to estimate true class probabilities $p_i^r(y_j)$ by: $p_i^r(y_j) = \sum_{k=1}^{N_{\text{class}}} p_i^r(\hat{y}_k) \cdot \condyh{j}{k} $. Note that each robot only receives a confusion matrix $C^r_i$ from the cloud which consists of
conditional probabilities $\condhy{j}{j}$ and not $\condyh{j}{j}$. Therefore, we still need to figure out a way to calculate $\condyh{j}{j}$. Due to space limits, 
we present the Bayesian update of $\condyh{j}{j}$ in the Appendix \ref{sec:calculating_conditional_probabilities}. 

%which is a simple application of Bayes Rule. 
% \vspace{-1.5em}
 \begin{minipage}[t]{0.50\textwidth}
  \begin{mdframed}[roundcorner=5pt]
  \parbox[t][3.5cm][t]{\linewidth}{
        \begin{subequations} \label{eq:oracleopt}
        \centering
        \textsc{Problem 1: Oracle Optimization}
        \begin{small}
        \begin{align}
        \min_{\action{r}{1} \ldots \action{r}{\ndevice} } & \kl{r+1} \label{eq:kl_oracle} \\
        \text{subject to: } & \action{r}{i}  \geq 0 ; \hspace{1mm} \forall \hspace{1mm} i = 1, \ldots, \ndevice \\ 
        &            1^T \cdot \action{r}{i} \leq  N_{\text{cache}} ; \hspace{1mm} \forall \hspace{1mm} i = 1, \ldots, \ndevice \label{eq:distrib_action_oracle} \\ 
        %\forall \hspace{1mm} i = 1, \ldots, \ndevice \label{eq:distrib_action_oracle} \\ 
        &            \clouddataset{r+1} = \clouddataset{r} + \sum_{i=1}^{\ndevice} \left( \dataobmatrix{i}{r} \action{r}{i}  \right)  \label{eq:dataset_oracle} \\ \notag
                            %   & N_{\text{cache}} \cdot \action{r}{i} \leq N_{r}^i \cdot \devicedatasetri ; \hspace{1mm} \forall \hspace{1mm} i = 1, \ldots, \ndevice \label{eq:valid_distrib_oracle} \\  \notag
        \end{align} 
        \end{small}
        \end{subequations}
  }
  \end{mdframed}%
 \end{minipage}%
\hfill
 \begin{minipage}[t]{0.50\textwidth}
  \begin{mdframed}[roundcorner=5pt]
  \parbox[t][3.5cm][t]{\linewidth}{
    \begin{subequations} \label{eq:indvopt}
    \centering
    \textsc{Problem 2: Greedy Optimization}
    \begin{small}
    \begin{align}
    \min_{\action{r}{i}} \hspace{0.2cm} & \kl{r+1} \\
    \text{subject to: } & \action{r}{i} \geq 0 \\
    & 1^T \cdot \action{r}{i} \leq N_{\text{cache}} \\
    &   \clouddataset{r+1} = \clouddataset{r} + \left( \dataobmatrix{i}{r} \action{r}{i}  \right) \label{eq:dataset_indv} \\ \notag
                        %   & N_{\text{cache}} \cdot \action{r}{i} \leq N_{r}^i \cdot \devicedatasetri
    \end{align}
    \end{small}
\end{subequations}
  }
  \end{mdframed}%
 \end{minipage}%

\vspace{-0.9em}

% \begin{minipage}{0.55\textwidth}
%  \parbox[t][4cm][t]{\linewidth}{
% \begin{subequations} \label{eq:oracleopt}
% \begin{mdframed}
% \centering
% \textsc{Problem 1: Oracle Optimization}
% \begin{align}
% & \min_{\action{r}{i} \ldots \action{r}{\ndevice} } \kl{r+1} \label{eq:kl_oracle} \\ \notag
% \text{ subject to: }  \\ 
% &            \action{r}{i}  \geq 0, \hspace{0.2cm} 1^T \cdot \action{r}{i} =  N_{\text{cache}} ; \hspace{1mm} \forall \hspace{1mm} i = 1, \ldots, \ndevice \label{eq:distrib_action_oracle} \\ 
% &            \clouddataset{r+1} = \clouddataset{r} + \sum_{i=0}^{\ndevice} \left( \action{r}{i}  \right)  \label{eq:dataset_oracle} \\ \notag
%                     %   & N_{\text{cache}} \cdot \action{r}{i} \leq N_{r}^i \cdot \devicedatasetri ; \hspace{1mm} \forall \hspace{1mm} i = 1, \ldots, \ndevice \label{eq:valid_distrib_oracle} \\  \notag
% \end{align}
% \end{mdframed}
% \end{subequations}
% }
% \end{minipage}
% \hfill
% \begin{minipage}{0.46\textwidth}
% \begin{subequations} \label{eq:indvopt}
% \begin{mdframed}
% \centering
% \textsc{Problem 2: Greedy Optimization}
% \begin{align}
% & \min_{\action{r}{i}} \hspace{0.2cm} \kl{r+1} \\ \notag
% \text { subject} & \text{ to: }  \\ 
% &   \action{r}{i} \geq 0, \hspace{0.2cm} 1^T \cdot \action{r}{i} =  N_{\text{cache}} \\
% &   \clouddataset{r+1} = \clouddataset{r} + \left( \action{r}{i}  \right) \label{eq:dataset_indv} \\ \notag
%                     %   & N_{\text{cache}} \cdot \action{r}{i} \leq N_{r}^i \cdot \devicedatasetri
% \end{align}
% \end{mdframed}
% \end{subequations}
% \end{minipage}
%\subsection{Greedy Action Policy}

\paragraph*{Greedy Action Policy: }
A greedy action policy, referred to as \Greedy, will not have any information about other robots' local data distribution, confusion matrix, or observed datasets. Thus, the best the robot can do is to attempt to minimize the loss function $\kl{r+1}$ by only optimizing its own action $\action{r}{i}$ \textit{individually}, as shown in Eq. \ref{eq:indvopt}. 
%We refer to this policy as \Greedy since each robot is only optimizing its action without considering data observed by other robots. 
The optimization program \ref{eq:indvopt} is very similar to that of the \Oracle \space policy (Eq. \ref{eq:oracleopt}), with the only difference being that the decision variables are reduced to one.
Since the \Oracle \space (Eq. \ref{eq:oracleopt}) and \Greedy \space (Eq. \ref{eq:indvopt}) policy optimization programs have a convex objective with linear constraints, they are guaranteed to converge to an optimal solution. 
% We discuss other special cases in the Appendix.

% \vspace{-1.0em}
% \vspace{-2.0em}

% \pohan{Also, there are two special cases when \Oracle and \Greedy are identical: 1. When the classifier models are perfect (confusion matrices are identity matrices). 2. the feasible space of each robot is orthogonal to each other's. See \ref{def:fspace} for definition.}

\section{A Cooperative Algorithm for Data Collection}
\label{sec:data_collection}
%In this work, we propose an algorithm for generating actions for each robot, which only requires interaction between the robots and no interaction with the cloud. Rather than the cloud calculating actions for each robot in one-shot, as shown in the \Oracle \space optimization program (\ref{eq:oracleopt}), each robot calculates its actions individually using shared information from other robots. Since all the robots share information, we refer to our algorithm as \Interactive.

We propose an \Interactive \space algorithm for generating actions for each robot, which only requires interaction between the robots and no cloud coordination. Rather than the cloud calculating actions for each robot in one-shot, as shown in the \Oracle \space optimization program (\ref{eq:oracleopt}), each robot calculates its actions individually using shared information from other robots. 
Importantly, each robot only shares its feasible action without divulging its confusion matrix or local data distribution to others. 

Alg. \ref{alg:train} describes our \Interactive \space policy, which runs for each round $r$. The inputs (line \ref{ln:inputs}), which are visible to each robot, are the target dataset $\datasettarget$ and the current cloud dataset $\clouddataset{r}$. 
We initialize each robot's action $\action{r}{i}$ in lines \ref{ln:for1_loop} - \ref{ln:for1_loop_end} using the \Greedy \space policy (Eq. \ref{eq:indvopt}) because the robots have not yet communicated any information about each others' tentative actions. 
In lines \ref{ln:share_init_actions} - \ref{ln:for2_loop_end}, we calculate optimal actions for each robot using the \Interactive \space message passing algorithm. 

%Each robot's action $\action{r}{i}$ aims to minimize the convex distance $\kl{r}$ between the current cloud and target dataset $\datasettarget$. 
%The key intuition behind \Interactive \space is that each robot will be able to contribute better (i.e. help in reducing the cloud dataset distance) if it is aware of others' actions.  

We start by sharing each robot's product of feasible data matrix and initial action (line \ref{ln:init_actions}) with all other robots (line \ref{ln:share_init_actions}). Then, we iterate over each robot (lines \ref{ln:for2_loop} - \ref{ln:for2_loop_end}) and calculate its best action $\action{r}{i}$ using the optimization program
Eq. \ref{eq:interactiveopt} while considering the other robots' actions fixed (line \ref{ln:interactive_opt}). The optimization program in Eq. \ref{eq:interactiveopt} is similar to that of the \Oracle \space policy (Eq. \ref{eq:oracleopt}); the difference lies in the calculation of the cloud dataset at round $r+1$ in Eq. \ref{eq:dataset_inter} and having one decision variable.

%Optimization program \ref{eq:interactiveopt} uses actions of other robots to calculate robot $i$'s best action. 

In line \ref{ln:share_calc_actions}, each robot shares its product of the feasible data matrix and the optimal action calculated using Eq. \ref{eq:interactiveopt} with the others. This repeats until our system reaches a Nash equilibrium (i.e. a fixed point, where no robot would change its action).  Finally, after convergence, we upload data from each robot sampled according to its final calculated action $\action{r}{i}$ (line \ref{ln:share_cache}). Since the \Interactive \space
optimization program \ref{eq:interactiveopt} is convex, it converges to an optimal solution (see Thm. \ref{theorem:converge_eventually}).
%has a convex objective with linear constraints, it converges to an optimal solution (see Thm. \ref{theorem:converge_eventually}).

\begin{minipage}{0.49\textwidth}
    \begin{small}
    \begin{algorithm}[H]
    \DontPrintSemicolon
    \SetNoFillComment % <---------------------------
    \SetAlgoLined
        % \textbf{Input:} $\datasettarget, \{ \devicedatasetri ; \hspace{1mm} i = 1, \ldots, \ndevice \} $ \; \label{ln:inputs} 
        \textbf{Input: } \text{Target, Cloud Dataset} $\datasettarget$, $\clouddataset{r}$ \; \label{ln:inputs} 
        \For{$i = 1, \ldots, \ndevice$ \label{ln:for1_loop} }
        {
            % Calculate class distributions $\devicedatasetri$ using $\devicedataset{r}{i}$. \; \label{ln:calc_class_distib}
            Initialize $\action{r}{i}$ using \Greedy \space actions Eq. \ref{eq:indvopt}. \; \label{ln:init_actions}
        } \label{ln:for1_loop_end}
        Share $\dataobmatrix{i}{r} \action{r}{i}$ with all robots. \;  \label{ln:share_init_actions}
        \While{Not Converged \label{ln:while_loop}}
         {   
            \For{$i = 1, \ldots, \ndevice$ \label{ln:for2_loop} }
            {
                Get action $\action{r}{i}$ using opt. program Eq. \ref{eq:interactiveopt} \; \label{ln:interactive_opt}
                Share actions $\dataobmatrix{i}{r}\action{r}{i}$ with all robots. \;  \label{ln:share_calc_actions}
            } \label{ln:for2_loop_end}
        } \label{ln:while_loop_end}
        \For{$i = 1, \ldots, \ndevice$ \label{ln:for3_loop} }
        {
            Upload caches determined by actions $\action{r}{i}$ \; \label{ln:share_cache}
        } \label{ln:for3_loop_end}
        % \vspace{1em}
        \caption{\textbf{\Interactive \space Algorithm}}
     \label{alg:train}
    \end{algorithm}
    \vspace{-1em}
    \end{small}
\end{minipage}
\begin{minipage}{0.49\textwidth}
    \begin{subequations} \label{eq:interactiveopt}
    \begin{mdframed}
\centering
\textsc{Problem 3: Interactive Optimization}
    \begin{small}
    \begin{align}
        \min_{\action{r}{i}} \hspace{0.2cm} & \kl{r+1} \\
        \text {subject to: } & \action{r}{i} \geq 0 \\
        & 1^T \cdot \action{r}{i} \leq N_{\text{cache}} \\
        & \clouddataset{r+1} =  \clouddataset{r} +  \sum_{k=1; k \neq i}^{\ndevice} \left( \dataobmatrix{k}{r} \action{r}{k}  \right) + \left( \dataobmatrix{i}{r} \action{r}{i} \right)  \label{eq:dataset_inter} %\notag
                            %   & N_{\text{cache}} \cdot \action{r}{i} \leq N_{r}^i \cdot \devicedatasetri
    \end{align}
    \end{small}
    \end{mdframed}
    \end{subequations}
\end{minipage}

\paragraph*{Theoretical Analysis: }
We first show that the \textit{while} loop (lines \ref{ln:while_loop} - \ref{ln:while_loop_end}) in our proposed Alg. \ref{alg:train} will eventually converge. Moreover, we provide easily-obtained conditions for when it converges in \textit{one iteration}, which minimizes inter-robot communication. Crucially, we also show that our interactive policy matches the omniscient oracle policy.  
\textbf{All proofs} are in the Appendix \ref{sec:converge_eventually} - \ref{sec:total_num_messages}.
%\textbf{Proof:} provided in Appendix: \ref{sec:one_iteration}.

\begin{theorem}[Convergence] \label{theorem:converge_eventually}
    The \textit{while} loop (lines \ref{ln:while_loop} - \ref{ln:while_loop_end}) in Alg. \ref{alg:train} will eventually converge.
\end{theorem}
%\textbf{Proof:} provided in Appendix \ref{sec:converge_eventually}.
Next, we show one of the main technical contributions of this paper, which states that our proposed \Interactive \space algorithm will reach the same optimal solution as the \Oracle \space upon termination. 

%Next, we show that our proposed \Interactive \space action policy (Alg. \ref{alg:train}) will converge to the \Oracle \space action policy upon termination. This is one of the main technical contributions of this paper, which shows that our proposed algorithm will reach the same optimal solution as the oracle. 

\begin{theorem}[\Interactive \space converges to \Oracle] \label{theorem:orable_inter_same}
    The while loop in Alg. \ref{alg:train} lines \ref{ln:while_loop} - \ref{ln:while_loop_end} is guaranteed to return an action (denoted by $\action r {int,i}$) that is equal to the \Oracle \space action denoted by $\action r {o,i}$.
    \end{theorem}
%\textbf{Proof:} provided in Appendix \ref{sec:oracle_inter_same_sec}.
Next, we provide practical conditions for when our proposed \Interactive \space action policy will converge in \textit{one iteration} of message passing, which bounds inter-robot communication. 

%when the total number of uploaded data-points is smaller than the difference between the size of target dataset $\mathcal{D}_{\text{target}}$ and the size of current cloud dataset $\clouddatasetr$, namely
%$\mathbf{1}^\top(\mathcal{D}_{\text{target}}-\clouddatasetr) > \ndevice \times \ncache$. 

%Finally, we show that our proposed \Interactive \space action policy will converge in one iteration of message passing when the total number of uploaded data-points is smaller than the difference between the size of target dataset $\mathcal{D}_{\text{target}}$ and the size of current cloud dataset $\clouddatasetr$, namely
%$\mathbf{1}^\top(\mathcal{D}_{\text{target}}-\clouddatasetr) > \ndevice \times \ncache$. This condition holds for all rounds except for the last round that reaches the target distribution and data collection terminates.

\begin{theorem}[Bounded Communication] \label{theorem:converge_one}
When the total number of uploaded data-points is smaller than the difference between the size of target dataset $\mathcal{D}_{\text{target}}$ and the current cloud dataset $\clouddatasetr$, namely
$\mathbf{1}^\top(\mathcal{D}_{\text{target}}-\clouddatasetr) > \ndevice \times \ncache$, the while loop in Alg. \ref{alg:train} lines \ref{ln:while_loop} - \ref{ln:while_loop_end} terminates in one iteration.
\end{theorem}

%\textbf{Proof:} provided in Appendix: \ref{sec:one_iteration}.

The condition in Thm. \ref{theorem:converge_one} holds for all rounds except for the last round that reaches the target distribution, upon which data collection terminates. All our theory assumes that all actions $\action{r}{i} \in \mathbb{R}^{N_{\text{class}}}$ can realize any feasible real-valued vector. However, in reality, we will only have an integer-valued action vector since we can only upload a discrete set of images, which becomes an integer programming problem.
However, for real-world datasets with thousands of images, we can just round the continuous solution
to get a very close approximation to the (generally intractable) integer case.

%\begin{proposition}[Total Number of Messages] \label{prop:total_msgs}
%    The total number of messages passed between the robots in line \ref{ln:share_init_actions} and in each for loop iteration (lines \ref{ln:for2_loop}-\ref{ln:for2_loop_end}) is $\ndevice^2-\ndevice$, respectively.
%\end{proposition}

% Specifically, we prove that each action computation in line \ref{ln:interactive_opt} will strictly reduce the loss function, $\kl{r+1}$, and while loop, lines \ref{ln:while_loop} - \ref{ln:while_loop_end}, will converge in a single iteration.

%\section{Experiments}
\section{Experiments and Conclusion}
We now compare our Alg. \ref{alg:train} with benchmark methods on four diverse datasets. The first two datasets of \texttt{MNIST} \cite{lecun2010mnist} and \texttt{CIFAR-10} \cite{Krizhevsky09learningmultiple} serve as proof-of-concepts for the domains of handwritten digit and common object classification. Then, we use the \texttt{Adverse-Weather} dataset \cite{adverer2022umich}, which contains tens of thousands of images to train self-driving vehicles to classify rain, fog, snow, sleet, overcast, sunny, and cloudy
driving conditions. To show the generality of our theory, we then extend to the state-of-the-art Berkeley Deep Drive (\texttt{DeepDrive}) dataset \cite{yu2020bdd100k}, which has $100$K images of various weather conditions and road scenarios for self-driving cars. 

\textit{Comparison Metric:} To compare all methods, we use the $L_2$-norm (the optimization objective) between the target $\datasettarget$ and the current cloud dataset $\clouddataset{r}$. For statistical confidence, all experiments are repeated for more than $10$ times with different random seeds that capture uncertainty in sampling from the confusion matrix $C^r_i$ and observing different distributions of local data per robot. Further experiment parameters are detailed in the Appendix \ref{sec:exp_setup}. We compare the following methods:

\begin{enumerate}[leftmargin=*,noitemsep,topsep=0pt]
    \item \Greedy \space solves the optimization program in Eq. \ref{eq:indvopt} \textit{individually} per robot by minimizing the $L_2$-norm between the target and cloud distribution without information about other robots. 
        %by only considering that specific robot's class data distribution $\devicedatasetri$.
    \item \Oracle \space solves the optimization program in Eq. \ref{eq:oracleopt}. It perfectly knows all incoming class data distributions $p^r_i(y)$ and confusion matrices $C^r_i$ for all robots $i$ and thus calculates the optimal action for each robot in one common optimization problem (Eq. \ref{eq:oracleopt}). 
        %Since this policy has complete information about all robots, it is considered an \textit{Oracle} policy.
    \item \Uniform \space is a deterministic policy which assigns the same probabilities to all classes for each robot, i.e. $\action{r}{i} = \left[ \frac{1}{\nclass} \ldots \frac{1}{\nclass} \right]$. It represents a simple heuristic for equally sampling all classes.
    \item \Lowerbound \space (derived in Lemma \ref{lemma:lowerbound}) is the lower bound of the objective function of the \Oracle\space policy for a given target dataset, $\datasettarget$, current cloud dataset, $\clouddataset{r}$, and local data distribution $p^r_i(y)$. It represents how well can sample in the absence of perceptual uncertainty.
        %It represents how well can sample if each action is feasible (i.e., no perceptual uncertainty).
        %It is independent of all robots' feasible matrices. 
        %It is derived in Appendix Lemma \ref{lemma:lowerbound}.
    \item \Interactive \space runs our Alg. \ref{alg:train}. 
        It is not an \textit{Oracle} policy, as it only shares the action taken by other robots and not the actual class data distribution, $p^r_i(y)$, nor the confusion matrix.
\end{enumerate}

\begin{figure*}[t]
    \includegraphics[width=1.0\columnwidth, height = 16cm]{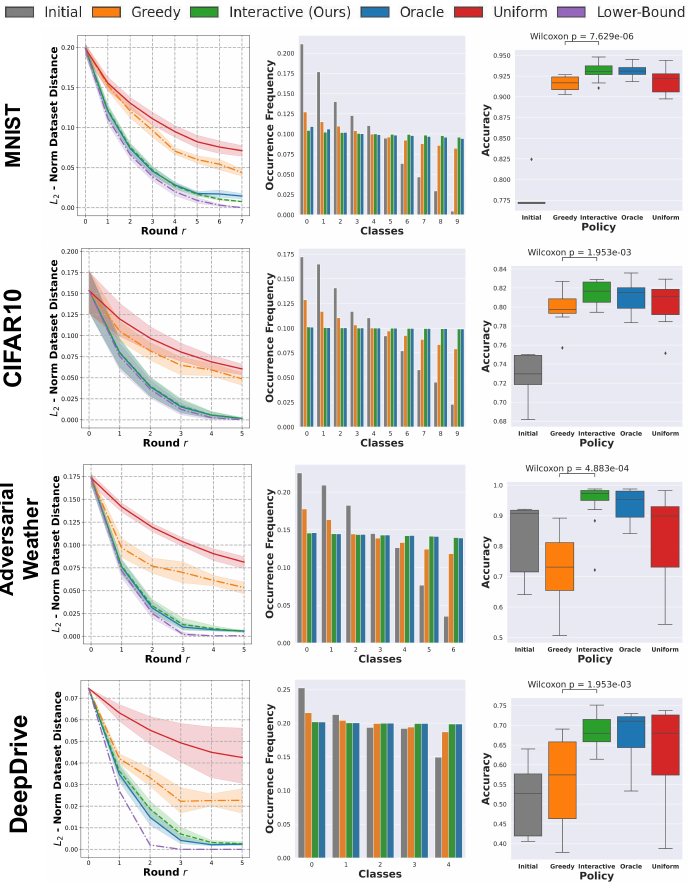}
    \caption{\small{\textbf{Our game-theoretic \Interactive \space policy outperforms benchmarks and converges to the \Oracle.} Each row is a different dataset. \textit{Column 1:} As expected by our theory, \Interactive \space minimizes the $L_2$-norm distance (optimization objective, y-axis) better than \Greedy \space and matches the omniscient \Oracle. \textit{Column 2:} Clearly, \Interactive \space achieves a much more balanced distribution of classes
    (target distribution is uniform) than benchmarks. \textit{Column 3:} Since \Interactive \space achieves a more balanced dataset, this experimentally translates to a higher DNN accuracy (statistically significant) on a held-out test dataset. }}
     \label{fig:combined_metrics}
\vspace{-1em}
\end{figure*}

%\paragraph*{Results: }
\textbf{Results: }
Our experimental results (Fig. \ref{fig:combined_metrics}) demonstrate that our proposed \Interactive \space policy performs as well as the \Oracle \space, as proved in Thm. \ref{theorem:orable_inter_same}. Additionally, we demonstrate that our \Interactive \space policy is much better than the \Greedy \space and \Uniform \space policies on all datasets. Finally, we show that no action policy can perform better than the derived \Lowerbound \space action policy. %Fig. \ref{fig:combined_metrics} shows the performance of all methods on all datasets. Each row corresponds to a dataset and each column investigates a central question, as detailed below. 
\newpage
\textbf{\textit{Does cooperation minimize distance to the target data distribution?}}

Our optimization objective is to minimize the $L_2$-norm distance between the cloud dataset and target data distribution, which we plot in the first column. 
Clearly, our \Interactive \space policy significantly outperforms the \Greedy \space and \Uniform \space policies on all datasets. 
Specifically, we beat the \Greedy \space policy by $23.6\%, 44.8\%, 40.3\%, 38.7\%$ in $L_2$-norm distance on the \texttt{MNIST}, \texttt{CIFAR-10}, \texttt{Adverse-Weather}, and \texttt{DeepDrive} datasets respectively. These benefits arise because the \Interactive \space policy allows robots to coordinate the rare classes they upload, but the \Greedy \space policy might lead to uncoordinated uploading of
redundant data. Moreover, our \Interactive \space policy performs nearly identically as the \Oracle \space method, with small deviations due to imperfect vision models and randomness in local data distributions between trials. This is natural since we
proved that the \Interactive \space action policy reaches the same optimal value in expectation as the \Oracle \space policy in Thm. \ref{theorem:orable_inter_same}. Finally, we observe that no action policy outperforms our \Lowerbound \space policy derived in Lemma \ref{lemma:lowerbound}. 

\textbf{\textit{Does cooperation achieve more balanced datasets?}}

In column two, we see that the initial data distribution among robots (gray) is highly imbalanced since they operate in diverse contexts. However, we see that our \Interactive \space policy (green) achieves a much more balanced dataset distribution compared to \Greedy \space (orange), which is natural since the convex objective minimizes the distance to a uniform distribution. %We now show the benefits of such a balanced dataset for re-training robust ML models. 

\textbf{\textit{What is the final accuracy of trained models?}}

In column 3, we show the final accuracy of re-training DNN classification models on the datasets accrued by each method in the cloud. Importantly, our proposed \Interactive \space action policy leads to better accuracy gains than the \Greedy \space and \Uniform \space action policies. We beat the \Greedy \space policy by $1.4\%, 1.7\%, 21.9\%, 12.4\%$ in accuracy on the \texttt{MNIST}, \texttt{CIFAR-10}, \texttt{Adverse-Weather}, and \texttt{DeepDrive} datasets respectively. This is because the \Interactive \space action policy makes sure we collect classes lacking in the current cloud dataset, thus preventing class-imbalance issues in model training. While our theory only addresses
convex distances between dataset distributions (column 1 and 2), we show strong experimental results for re-training non-convex DNN classifiers. 
The \Interactive \space and \Oracle \space algorithms lead to slightly different final accuracies since they can potentially upload a different set of images and there is not a closed form relationship between the number of images and accuracy of a non-convex DNN. 
As detailed in the Appendix, \Interactive \space achieves very close to state-of-the-art accuracy for each dataset with only a limited set of uploaded datapoints. DNN architectures are also detailed in the Appendix. 
Collectively, these results closely align with our theory and show strong experimental benefits on real-world data. 
%Collectively, these results closely align with our theory and show strong experimental benefits for game-theoretic cooperation on diverse real-world datasets. 
%Further, the results closely align with our theoretical analysis.  

%showing that our \Interactive \space algorithm matches the performance of \Oracle \space and drastically outperforms \Greedy \space and \unifor
%demonstrate the efficacy of our proposed method on challenging real-world datasets. Furthermore, we experimentally show that our method, \Interactive, can perform as well as the \Oracle \space method by sharing the actions amongst the robots and not the complete information. Additionally, these experimental results align with Thm. \ref{theorem:orable_inter_same}. Finally, we show that our proposed method, \Interactive, also significantly outperforms the other methods of \Greedy \space and \Uniform \space for both datasets and action spaces.

\textbf{Limitations: } Our work assumes each robot can interact, which does not scale for extremely large \rebuttal{fleets}. Moreover, we assume that we sample images according to a classification model, even though we can train models for other tasks on the uploaded images. In future work, we aim to extend our theoretical guarantees for sub-clusters of communicating robots and cluster a continuous data distribution based on similar embeddings that serve as
virtual ``classes''. \rebuttal{Such an ability to generalize beyond discrete classes may enable our algorithm to scale to learning data-driven control policies.} 
%Finally, our theory only applies to data collection, not the final accuracy of a non-convex DNN, which is an open problem.  

%\section{Conclusion}
\textbf{Conclusion: }
This paper presents a theoretically-grounded, cooperative data sampling policy for networked robotic fleets, which converges to an oracle policy upon termination. Additionally, it converges in a single iteration under a mild practical assumption, which allows communication efficiency on real-world AV datasets. Our approach is a first step towards an increasingly timely problem as today's AV fleets measure terabytes of heterogenous data in diverse operating contexts \cite{intel}. 
In future work, we plan to develop policies that approximate the oracle solution when only a subset of robots can form coalitions and certify their resilience to adversarial node failures. 
%that each robot's potentially proprietary vision model remains private despite communication. 

%Moreover, our game-theoretic approach is appealing as 
%Our key insight is that networked robots mainly operate in heterogeneous environments, and therefore the robots can interact among themselves to choose an action that minimizes the distance of the cloud dataset from the target dataset. This insight leads us to formulate the data-offloading in networked robots as a cooperative convex game. 

%\textit{Limitations: } In future work, we plan to consider a more practical scenario where only a set of robots can interact rather than all robots and further analyze this scenario with a theoretical outlook. Additionally, we plan to explore tighter upper bounds on the performance of our proposed game-theoretic data-offloading action policy.

\section*{Acknowledgements}
This material is based upon work supported in part by Cisco Systems, Inc. under MRA MAG00000005. This material is also based upon work supported by the National Science Foundation under grant no. 2148186 and is supported in part by funds from federal agency and industry partners as specified in the Resilient \& Intelligent NextG Systems (RINGS) program.

\clearpage
\bibliography{ref/swarm, ref/external}

\clearpage
\section{Appendix}
The appendix is organized as follows: 
\begin{enumerate}
    \item Subsections \ref{app:whypotential} to \ref{sec:calculating_conditional_probabilities} describe the preliminaries.
    \item Subsection \ref{sec:exp_setup} explains the datasets, experimental parameters, and DNN architectures used in this work.
    \item Subsections \ref{sec:performance_gap} to \ref{sec:total_num_messages} give proofs for all theorems. 
\end{enumerate}

\subsection{Why A Potential Game?}

A potential function in a game is defined in Chapter 8 of \cite{tijs2003introduction}. 
It is a function indicating the incentives of all players (in our case robots), and any game with a potential is called a potential game. 
Typically, the goal of a player is to maximize its incentive expressed by the potential function. In our case, minimizing the loss function $\kl{r+1}$ in Eq. \ref{eq:dataset_inter} is the common goal for all robots, so the potential function is the negative of the loss function,  $-\kl{r+1}$. 
Note that the potential function $-\kl{r+1}$ is concave since the loss function $\kl{r+1}$ is convex.
\label{app:whypotential}

% Cooperative game is defined in \cite{brackin2002cooperative}. 

\subsection{Confusion Matrix} 

The confusion matrix of robot $i$ at round $r$ is defined as $C^{r}_{i}$, obtained by calculating the validation accuracy from a validation dataset.
Its $j$-th row represents the probability vector of classifying a data-point of class $j$ to different classes. 
If the confusion matrix is an identity matrix, it means that the classifier is perfect with $100\%$ accuracy. The matrix is:
\label{sec:confusion_matrix}

\begin{equation} \label{eq:confusion_mat}
\begin{aligned}
    C^{r}_{i} = \begin{bmatrix}
            \condhy{1}{1}  & \ldots & \condhy{N_{\text{class}}}{1} \\
            \condhy{1}{2}  & \ldots & \condhy{N_{\text{class}}}{2} \\
             \ldots   & \ldots &  \ldots \\
            \condhy{1}{N_{\text{class}}}  & \ldots & \condhy{N_{\text{class}}}{N_{\text{class}}} \\
          \end{bmatrix}.
\end{aligned}
\end{equation}

\subsection{Calculating the correct conditional probabilities} 
\label{sec:calculating_conditional_probabilities}

As mentioned in Sec. \ref{sec:formulation}, the robot's transmitted dataset $\action{r}{i}$ is calculated from the predicted class labels $\hat{y}_j$ and not the true class labels $y_j$, which are not available on-robot. However, we can use predicted class probabilities $p_i^r(\hat{y}_j)$ to estimate true class probabilities $p_i^r(y_j)$ by: $p_i^r(y_j) = \sum_{k=1}^{N_{\text{class}}} p_i^r(\hat{y}_k) \cdot \condyh{j}{k} $. 

We can obtain the conditional probability $\condyh{j}{j}$ by 
$\condyh{j}{j} = \frac{\condhy{j}{j} \cdot p^r_{i}(y_j)}{p_i^r(\hat{y}_j)}$ from the confusion matrix $C^{r}_{i}$ and $p_{i}^r(\hat{y}_j)$ can be calculated from the model inference on robot $i$. Note that $p^r_{i}(y_j)$ can be \textit{estimated} using a Bayesian Filter since we upload data at previous round $r-1$, which is assigned ground-truth labels. 

%Note that each robot only receives a confusion matrix $C^r_i$ from the cloud which consists of
%conditional probabilities $\condhy{j}{j}$ and not $\condyh{j}{j}$. Therefore, we still need to figure out a way to calculate $\condyh{j}{j}$. Due to space limits, 
%we present the Bayesian update of $\condyh{j}{j}$ in the Appendix \ref{sec:calculating_conditional_probabilities}. 

%\begin{assumption}
%The distribution of observed data for each robot, $i$, is constant for all rounds: $p_{i}^r(y) = p_{i}^{(r=1)}(y), \forall r$. $p_{i}^{(r=1)}(y)$ is the initial true label distribution for each robot $i$.
%\end{assumption} 
%
%This is a reasonable assumption if we assume a robot $i$ mostly operates in a similar environment across rounds,
%such as a robot in a sunny location. Naturally, the distribution for another robot $i'$ can be totally different in our setting, such as a robot in a snowy location. Hence, heterogeneity in data distributions is a hallmark of our setting. Therefore,  we can expect the true label distribution for any \textit{specific} robot $i$,  $p_{i}^r(y)$, to stay constant. For the ease of notation, we will represent $p_{i}^{(r=1)}(y)$ as $p_i(y)$. 

\subsection{Experiments}
\label{sec:exp_setup}

In the experiments, we simulated a system of multiple robots observing different image distributions $p_i^r(y)$ with the aim of sampling correct images to make the cloud dataset $\mathcal{D}_c^r$ as close as possible to the uniform target dataset $\mathcal{D}_{\text{target}}$. First, an initial dataset $\mathcal{D}_c^0$ with random class distributions is selected to train the initial classification model $f(x; \networkparam{0}_i)$. Next, the classification model is trained on the initial dataset $\mathcal{D}_c^0$ and its confusion matrix $C_{i}^0$ is calculated on the validation dataset. All robots have the same vision model in a simulation, thus the same confusion matrix. However, their incoming class distributions $p_i^r(y)$ are quite different, so each robot's feasible space $\fspace{i}{r}$ is different. In each round $r$, the robots label the images with their own classification model and solve the convex optimization problem to determine label allocations in the cache. At the end of each round $r$, sampled images are uploaded to the shared cloud, labeled by a human expert, and added to the cloud dataset with the correct labels. The cloud dataset statistics are updated and shared with all robots. When all sampling rounds are finished, the classification model is retrained on the final cloud dataset, which contains the initial dataset and sampled images from all robots. All the DNNs are written in PyTorch, and the \textsc{CVXPy} package is used to solve the convex optimization problem. 

In all experiments, we have divided our dataset into three non-overlapping parts: training, validation, and testing datasets. The training dataset is used to create the initial datasets and the images observed by robots. The validation and testing datasets are used to calculate the confusion matrix and the final accuracy, respectively. 

We now explain the datasets, the simulation parameters, training/validation/testing splits, and DNN training hyperparameters used in the simulations. 

\subsubsection{MNIST Dataset}

The \texttt{MNIST} dataset is a digit classification dataset consisting of 70,000
$28\times28$ grayscale images with ten classes. The dataset consists of 60,000 training images and 10,000 testing images. 

\paragraph{Simulation Parameters:}
For the \texttt{MNIST} dataset, we simulated $\ndevice = 20$ robots for 7 rounds each observing 2000 images and sharing only $\ncache = 2$ images with the cloud. The initial dataset size is set to $N_{\mathcal{D}_c^0} = 200 $ and at the end of the rounds a dataset of size $N_{\mathcal{D}_c^7} = 480 $ is accumulated.

\paragraph{Training, Validation, and Testing Split:}
We divided the original training dataset into training and validation datasets of sizes 54,000 and 6000, respectively, and used the original test dataset of size 10,000. Thus 0.3\% of the overall dataset size is used in the initial vision model. At the end of data sharing, we uploaded 0.8\% of the full \texttt{MNIST} standard dataset to train the vision model. Training a model on this dataset yields a final accuracy of 93.07\% on the full held-out test dataset. This value is close to 99.91\% state-of-the-art accuracy for the full training dataset, which is very good given that our scheme uploads only a \textit{fraction} of the data. 

\paragraph{DNN and Training Hyperparameters:}
We now describe the vision model for the classification task. A DNN with four convolutional layers and two fully connected layers with ReLU activation layers is used as the classification model. Between the convolutional layers, dropout is applied with a rate of 0.3. In the convolutional layers, a kernel with a filter size of (3,3) and stride of 1 are used with a padding of 1. Finally, fully connected layers with sizes of (128,10) are used in consecutive fully connected layers. When the models are trained, a learning rate of 0.01 is used, and the batch size is set to 1000. We used the ADAM optimizer in training and used the exponential learning rate scheduler with a decay rate of 0.99. The DNN models are trained for 200 epochs. We only normalized the images before inputting them into the classification model.

\subsubsection{CIFAR-10 Dataset}

The \texttt{CIFAR-10} dataset consists of 60000 $32\times32\times3$ RGB images with ten different classes. The original dataset is divided into training and testing datasets of sizes 50000 and 10000 respectively. This dataset contains 10 object classes: airplane, automobile, bird, cat, deer, dog, frog, horse, ship, and truck. 

\paragraph{Simulation Parameters:}
For the \texttt{CIFAR-10} dataset, we simulated $\ndevice = 20$ robots for 5 rounds, each observing 5000 images and sharing only $\ncache = 200$ images with the cloud. The initial dataset size is set to $N_{\mathcal{D}_c^0} = 10000 $ and at the end of the rounds a dataset of size $N_{\mathcal{D}_c^5} = 30000 $ is accumulated.

\paragraph{Training, Validation, and Testing Split:}
We divided the original training dataset into training and validation datasets of sizes 45,000 and 5000, respectively, and used the original test dataset of size 10,000. Thus 20\% of the overall dataset size is used in the initial vision model. At the end of data sharing, we uploaded 60\% of the full \texttt{CIFAR-10} dataset. Training a model on this dataset yields a final accuracy of 81.55\% on the full held-out test dataset. This value is comparable to 91.25\% state-of-the-art accuracy for
the full training dataset, which is very good given that our scheme uploads only a \textit{fraction} of the data. 
%and optimizes for a smaller number of epochs.

\paragraph{DNN and Training Hyperparameters:}
We now describe the vision model for the classification task. A ResNet32 Model with 32 convolutional layers and skip connections is used as the classification model. We didn't use any pretrained weights for the vision model. When the models are trained, a learning rate of 0.1 is used, and the batch size is set to 1000. We used the ADAM optimizer in training and used an exponential learning rate scheduler with a decay rate of 0.99. The DNN models are trained for 100 epochs. During training, we applied random cropping and random horizontal flips as data augmentation methods.

\subsubsection{Adversarial-Weather Dataset}

The \texttt{Adversarial Weather} dataset consists of thousands of $720\times1280\times3$ RGB image sequences collected in various weather conditions from moving vehicles. Most of the sequences are dynamic, while some are static recordings. The classes included in the dataset are rain, fog, snow, sleet, overcast, sunny, and cloudy. These weather conditions were recorded at various times of the day: morning, afternoon, sunset, and dusk. For the simulations, we have combined the time of day labels
and the weather labels and created a total of 7 classes. Since the images are created from video sequences, we have subsampled the images once in every five frames to prevent having similar images. In the end, we created a dataset with 46025 images.

\paragraph{Simulation Parameters:}
For the \texttt{Adversarial Weather} dataset, we simulated $\ndevice = 10$ robots for 5 rounds, each observing 2000 images and share only $\ncache = 20$ images with the cloud. The initial dataset size is set to $N_{\mathcal{D}_c^0} = 1000 $ and at the end of the rounds a dataset of size $N_{\mathcal{D}_c^5} = 2000 $ is accumulated.

\paragraph{Training, Validation, and Testing Split:}
We divided the  dataset into training, validation, and test datasets of sizes 37279, 4143, and 4603, respectively. At the end of data sharing, training the model on the accumulated dataset yields a final accuracy of 94.27\% on the full held-out test dataset. 

\paragraph{DNN and Training Hyperparameters:}
A ResNet18 Model with 18 convolutional layers and skip connections is used as the classification model. We initialized the model weights with weights pre-trained on the \texttt{ImageNet} dataset and updated all layers. During training, a learning rate of 0.1 is used, and the batch size is set to 128. We used the ADAM optimizer during training and used the exponential learning rate scheduler with a decay rate of 0.99. The DNN models are trained for 50 epochs. During training, we first downsampled the images to a size of $256\times455\times3$, and applied random cropping and random horizontal flips as data augmentation methods.

\subsubsection{DeepDrive Dataset}

The \texttt{DeepDrive} dataset with 100,000 images is a driving video dataset from various cities in different weather conditions.  The dataset consists of 70000 training, 10000 validation, and 20000 testing images. However, the testing images aren't publicly available. Therefore, we only used original training and validation datasets. We used the weather labels as the target of the classification model. The weather labels included in the datasets are: rainy, snowy, clear, overcast, partly cloudy, and foggy. We had to discard the foggy classes from the simulations because this class included only 181 images. Therefore, we trained the classification model on 5 classes. 

\paragraph{Simulation Parameters:}
For the \texttt{DeepDrive} dataset, we simulated $\ndevice = 20$ robots for 5 rounds, each observing 5000 images and sharing only $\ncache = 50$ images with the cloud. The initial dataset size is set to $N_{\mathcal{D}_c^0} = 8000 $ and at the end of the rounds the dataset of size $N_{\mathcal{D}_c^5} = 13000 $ is accumulated.

\paragraph{Training, Validation, and Testing Split:}
We divided the original training dataset into training and validation datasets of sizes 36968, 24646, respectively, and used the original validation dataset as the testing dataset of size 8830. Thus 11.43\% of the overall dataset size is used in the initial vision model. At the end of data sharing, we uploaded 18.57\% of the full \texttt{DeepDrive} dataset. Training a model on this dataset yields a final accuracy of 68.10\% on the full held-out test dataset. This value is
comparable to 81.57\% state-of-the-art accuracy for the full training dataset, which is very good given that our scheme uploads only a very small fraction of the data. Moreover, our scheme beats the Greedy Benchmark by 12.4\% as shown in Fig. \ref{fig:combined_metrics}.

\paragraph{DNN and Training Hyperparameters:}
A ResNet18 Model with 18 convolutional layers and skip connections is used as the classification model. We initialized the model weights with weights pre-trained on the \texttt{ImageNet} dataset and updated all layers. During training, a learning rate of 0.1 is used, and the batch size is set to 128. We used the ADAM optimizer in training and used the exponential learning rate scheduler with a decay rate of 0.99. The DNN models are trained for 50 epochs. During training, we first downsampled the images to a size of $256\times455\times3$, and applied random cropping and random horizontal flips as data augmentation methods.

\rebuttal{
\subsubsection{Heterogenous Data Distributions for Robots}
\label{subsec:non_IID}
We now show that all experiments have heterogenous data distributions. This is a hallmark of real robotics settings that we observed from the real AV datasets. Fig. \ref{fig:agents_dist} illustrates that each agent (x-axis) has a markedly different data distribution of true class probabilities $p_i^r(y)$ than others (y-axis barplot). For the synthetic \texttt{MNIST} and \texttt{CIFAR-10} datasets, we randomly shuffled data distributions across agents. However, we
plot the \textit{true} data distributions across AVs for the real-world autonomous driving datasets.

%In all experiments independent, nonidentical true class probabilities $p_i^r(y_j)$ are selected for the robots as shown in Fig. \ref{fig:agents_dist}. 

% REBUTTAL PLOT
\begin{figure*}[t]
    \includegraphics[width=1.0\columnwidth]{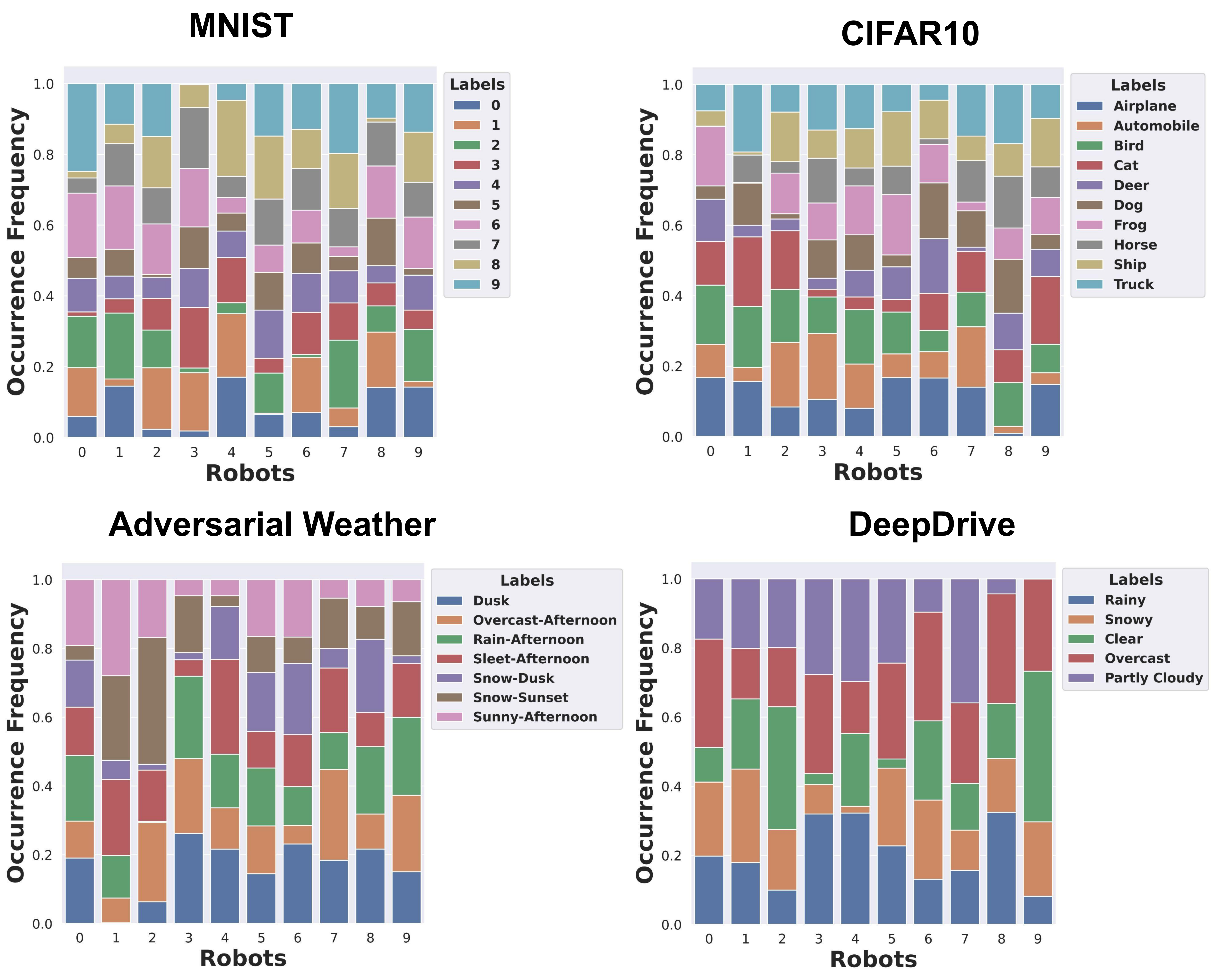}
    \caption{\rebuttal{\small{\textbf{ The true class probabilities $p_i^r(y_j)$ of 10 randomly selected agents for experiments is non-uniform:} As expected for real-life robotics settings, different robots observe non-identical, skewed data distributions in our experiments. We randomly shuffled data for the synthetic datasets (MNIST/CIFAR). However, we plot \textit{real-world} data distributions observed in the AV datasets. We randomly selected 10 agents for visual clarity.    }}}
     \label{fig:agents_dist}
\vspace{-1em}
\end{figure*}

\begin{figure*}[t]
    \includegraphics[width=1.0\columnwidth]{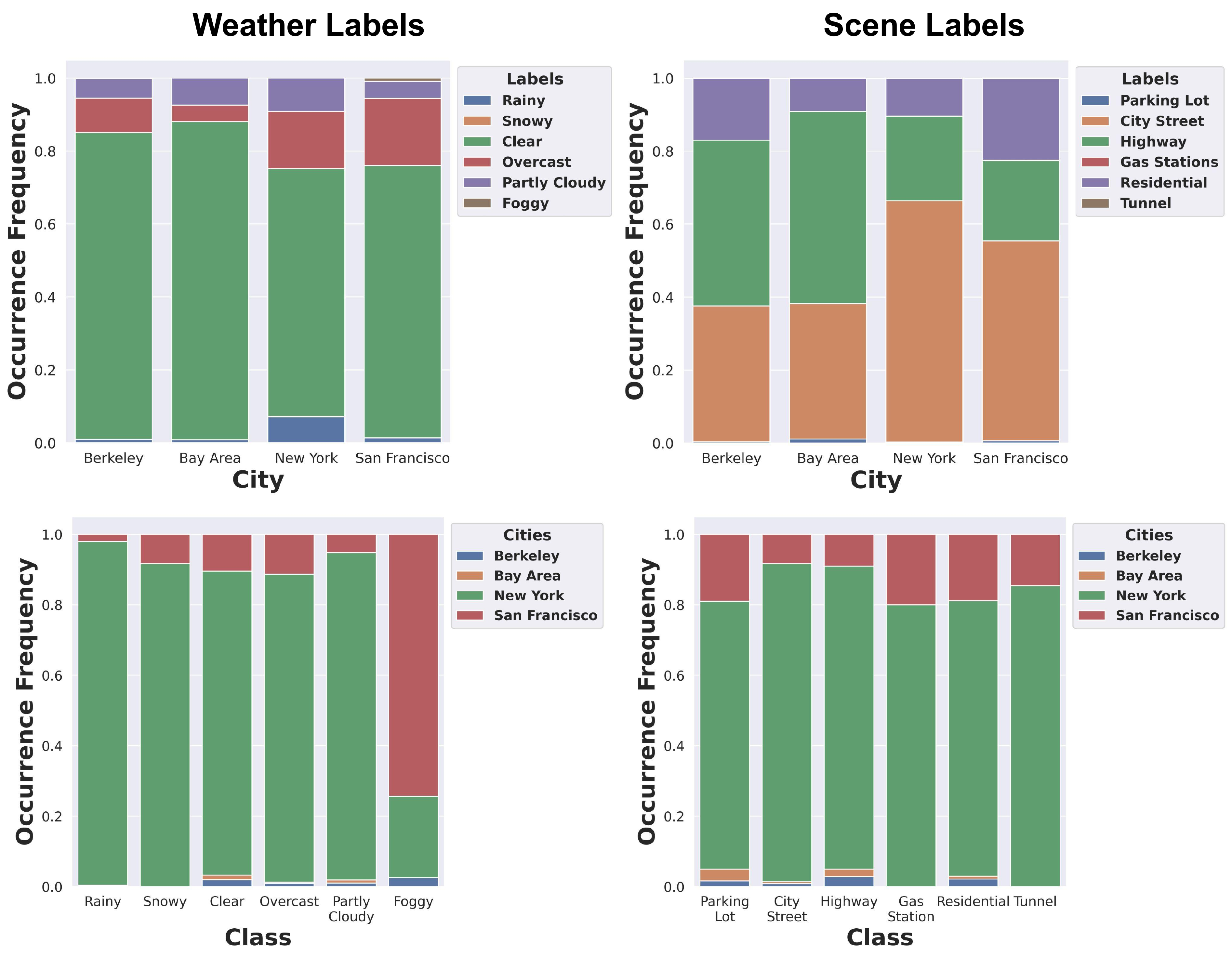}
    \caption{\rebuttal{\small{\textbf{ Heterogeneous Data Label Distributions Across Cities.}
    Even if we group across a full city, the cities differ in their label distributions. For example, a majority of scenes with rain occur in New York (left), while a majority of scenes with fog occur in SF. Thus, this paper's algorithms to coordinate data collection in heterogenous environments are needed for real-world AV datasets. 
    }}}
     \label{fig:jointbarplot}
\vspace{-1em}
\end{figure*}

\begin{figure*}[t]
    \includegraphics[width=1.0\columnwidth]{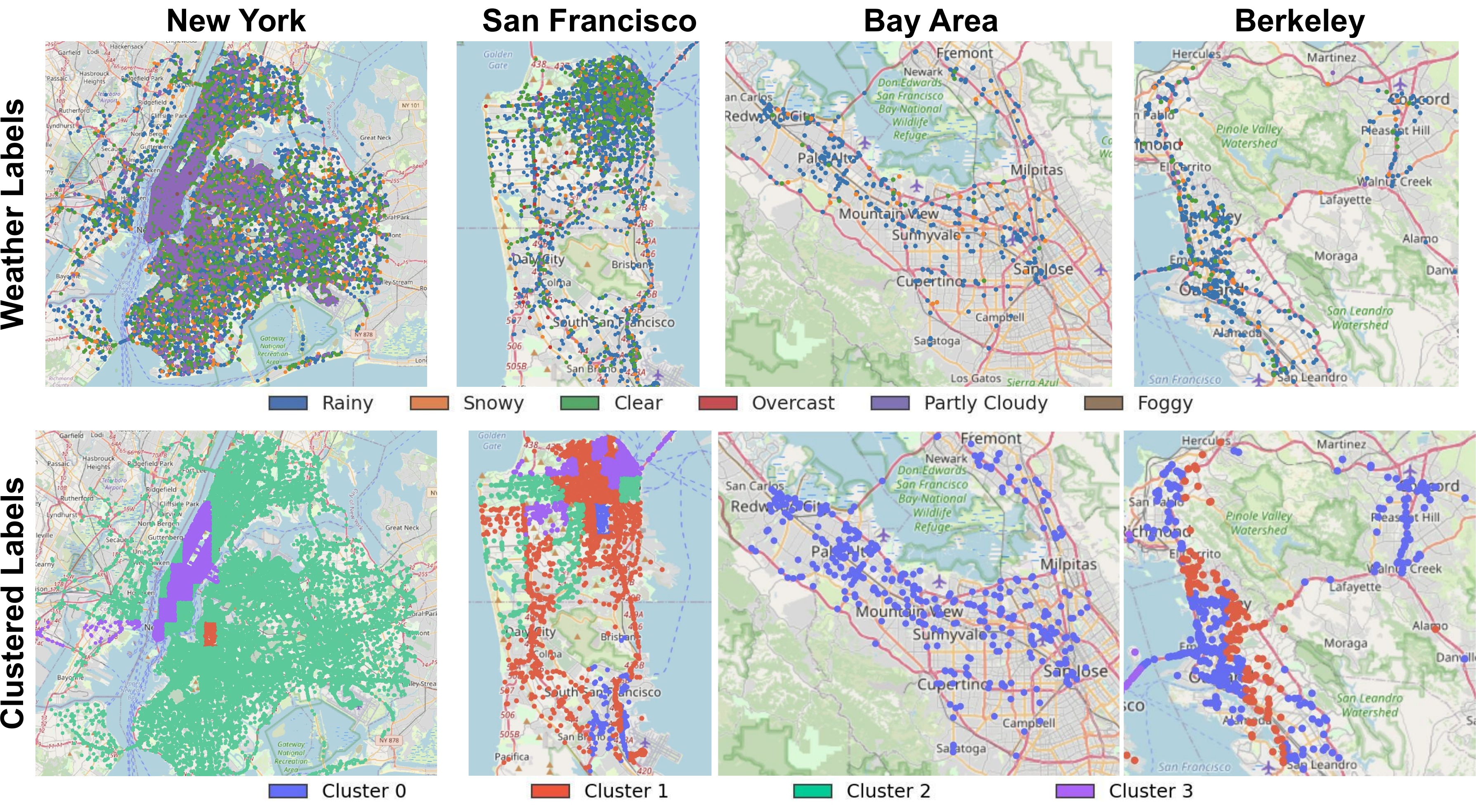}
    \caption{\rebuttal{\small{\textbf{Heterogenous data in real world AV datasets.
    } Top: The distribution of weather across cities is skewed, with much of the rain in the \texttt{DeepDrive} dataset in New York. Bottom: First, we group regions in a city and obtain a frequency distribution over classes in that area. Then, we perform k-means clustering with 4 clusters of the frequency distributions. Clearly, the frequency distributions of different weather conditions are similar locally, but very different across regions of a city.
    }}}
     \label{fig:cities_dist}
\vspace{-1em}
\end{figure*}

Fig. \ref{fig:jointbarplot} shows the distribution of classes across different cities in the \texttt{DeepDrive} dataset is also heterogeneous. Clearly, a majority of rainy scenes are in New York and a majority of foggy scenes are in SF. Moreover, New York has the highest percentage of city street scenes. 

Finally, the highly heterogenous distribution of classes can be seen in the map views from the \texttt{DeepDrive} dataset in Fig. \ref{fig:cities_dist}. Clearly, the majority of rainy points occur in NYC, especially lower Manhattan (top row). In the bottom row, we first divide the city into regions of a few miles and then create a probability distribution over the classes that appear in that sector. Then, we cluster these probability distributions into 4 meta-classes/clusters using
k-means. Clearly, close by
areas of a city have similar probability distributions over classes, but they are quite distinct in different geographic areas.

%\subsubsection{Practical Hardware Implementation on Nvidia Jetson Nano Embedded GPUs}
%\label{subsec:jetson_nano}
%We now demonstrate that our algorithm is efficient and scalable to run on low-power, compute-constrained robots, such as a fleet of factory robots. 
%Since we don't have a fleet of robots, we implemented our algorithm on 3 communicating Nvidia Jetson Nano GPUs connected over WiFi, which is a standard GPU used on low-power robots.
%We performed the same adversarial weather prediction task in Figure 3 of the original paper. We simply re-played the original video stream on each Nvidia Jetson Nano embedded GPU and ran our data collection algorithm and convex solver in real-time. This setting emulates a real robot observing video through its camera feed. Thus, the accuracy and convergence are exactly as in Figure 3.
% 
%Table \ref{fig:hardware_times} shows the computer vision models  are fast for inference (column 1) and the computation time for running the convex solver in \textsc{CVXPY} is also fast. Further, the communication time is also relatively low for our distributed data collection task, which is not real-time and safety critical.
%
%
%\input{fig_latex/hardware_times}
%\end{comment}

\clearpage
\subsubsection{Visualizing Heterogenous Data Distributions Across Space and Time}

\begin{figure}[H]
    \includegraphics[width=1.0\columnwidth,height=18cm]{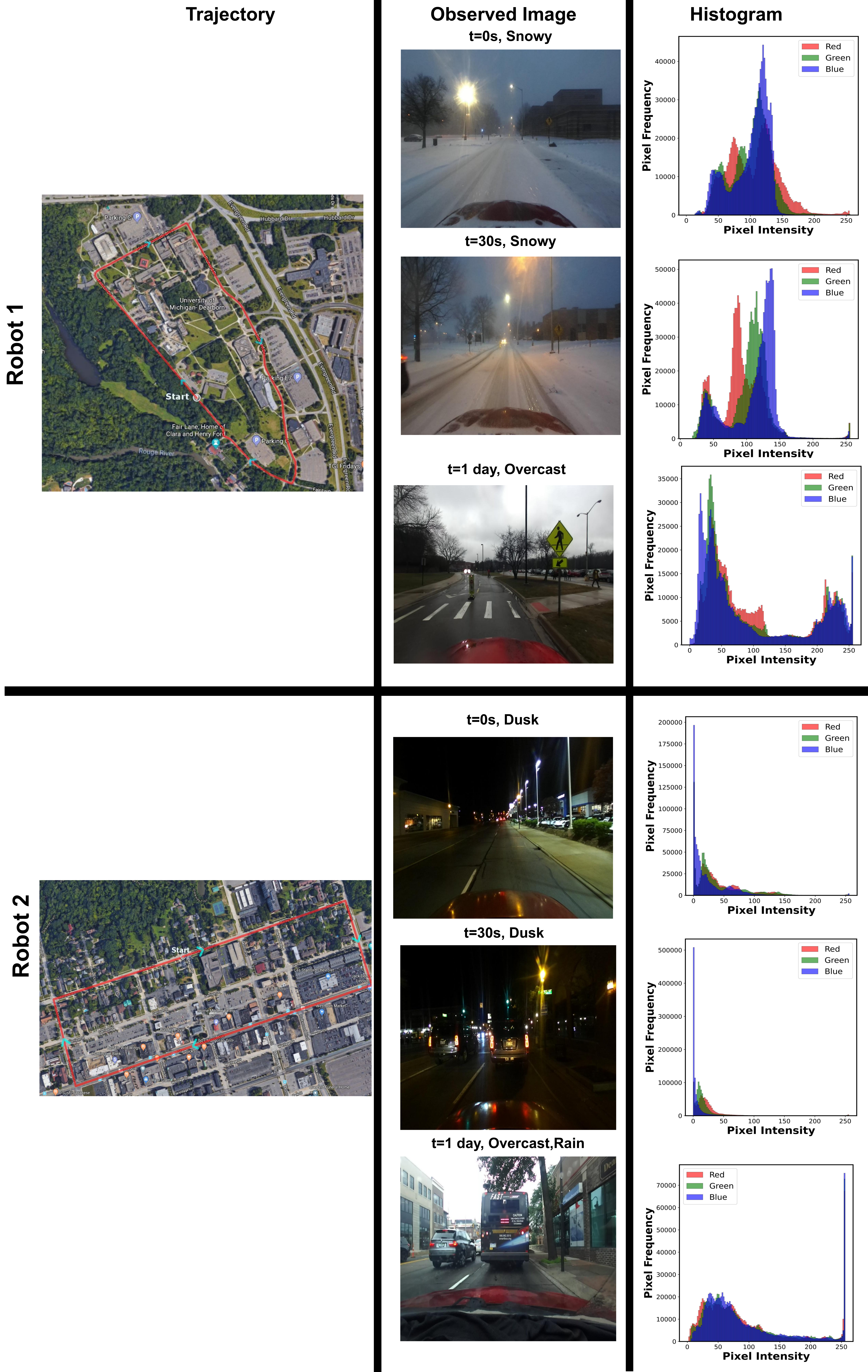}
    \caption{\small{\textbf{Observed Image Statistics Differ Across Space and Time for Real-World Robot Trajectories.} We show two robots’ trajectories on a map (only 2 for visual clarity). The robots operate in different parts of a city and observe different images - Robot 1 observes snowy and overcast images whereas Robot 2 observes Dusk and Overcast/Rain images. For all robots, we see that closely-spaced frames (30 seconds apart) have the same class. But even then, the pixel values
    are different from each other, as seen by the histogram of pixel intensities being different. For the same robots, randomly selected images from 1 day later have very different classes and pixel value distributions.  Also, when we compare different locations we see that the pixel distributions vary significantly.}}
     \label{fig:robot_imgs}
\vspace{-1em}
\end{figure}

\begin{figure*}[ht]
    \includegraphics[width=1.0\columnwidth]{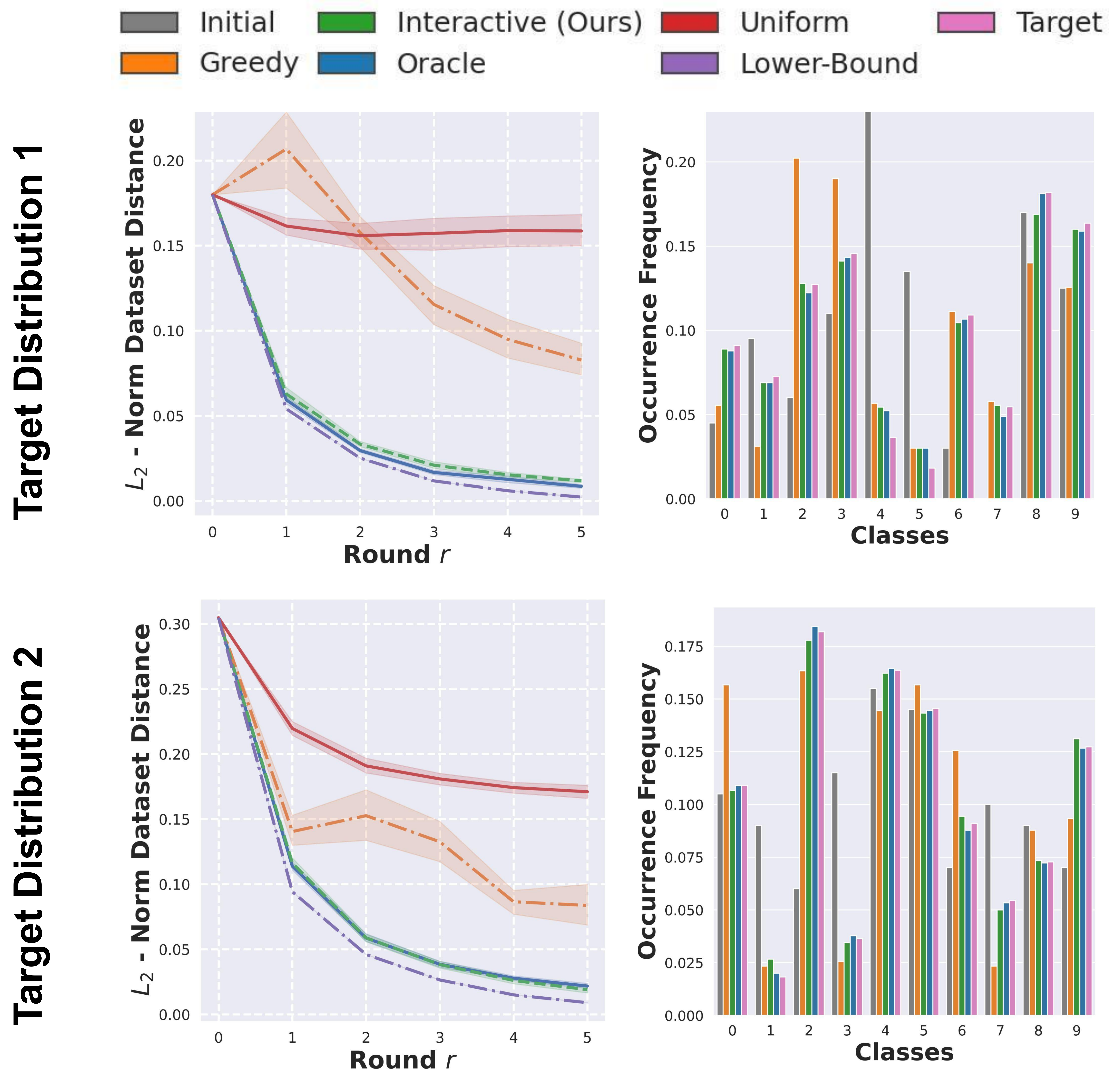}
    \vspace{1em}
    \caption{\small{\textbf{Our \Interactive \space policy quickly converges to non-uniform target data distributions.} Each row shows an independent experiment with a different, non-uniform desired target distribution set by the roboticist (pink). In the barplot, we see that the initial distribution of classes seen by heterogeneous robots is skewed. Clearly, our Interactive policy (green) quickly converges to the desired non-uniform target data distribution (left panel). Moreover, the
    \textit{final} data distribution achieved by our Interactive policy (green) closely matches the desired target distribution (pink) in the right barplots. Finally, we significantly outperform the heuristic benchmarks of greedy and uniform sampling. Thus, we verify that our algorithm works even when the target distribution is not uniform but rather any arbitrary target distribution. This is guaranteed by our theory since the target distribution is fixed and we have a convex loss function
    penalizing the difference between the current cloud distribution and the target (see Theorems 1-3).}}
     \label{fig:skewed_plots}
\end{figure*}

We now show that the distribution of observed image pixels and classes change for a robot across its trajectory (space and time). Fig. \ref{fig:robot_imgs} shows this heterogeneity across several real-world trajectories on a map. These complement the computed aggregate statistics in Appendix Figs. 4-6. We now describe visualize the heterogenous data distributions robots see in real-world autonomous driving datasets. 

\textbf{\textit{Any Given Robot Sees Different Pixels/Classes As it Moves Along its Trajectory (Space + Time) :}}

Fig. \ref{fig:robot_imgs} shows two robots’ trajectories on a map (only 2 for visual clarity). The robots operate in two different parts of a city in the Adversarial Weather dataset within different environments. Robot 1 observes mostly snowy and overcast images, whereas Robot 2 observes dusk and overcast images. For any two randomly selected frames close together (30 seconds apart), the images are temporally correlated so are not independent. Further, for the same robot, randomly selected images 1 day later have totally different classes (the scene transitions from snow to overcast or dusk to rain). Therefore, the classes are not identically distributed across space nor across time. We observed this for hundreds of robots and randomly selected 2 for visual clarity. 

\textbf{\textit{Aggregate Statistics Across Real-World Driving Datasets:}} 

Next, we show such heterogeneity across a full dataset. We do not shuffle these datasets artificially - they are the original, naturally-occurring, real world adversarial weather and DeepDrive datasets. For each robot, we show that the distribution of classes it sees across its trajectory is very different from robot to robot in Appendix Figs. 4-6.

\textbf{\textit{Histogram of Pixel Differences: }} 

Our algorithm targets distributed collection of diverse classes of data. Since these robots observe diverse real-world images across space and time, naturally the distribution of raw pixels will be different. To prove this, we compute a histogram of pixel color values in the right panel of Fig. \ref{fig:robot_imgs}. First, in Fig. \ref{fig:robot_imgs}, we show that even for the same robot, the distribution is different for two randomly selected images which are 30 seconds apart. Then, in Fig.\ref{fig:robot_imgs} column 3, we compare the distribution of pixel intensities for different robots and different classes, which are indeed different. 

\subsubsection{Convergence to Non-Uniform Target Data Distributions}

We now illustrate that our algorithm can converge to any non-uniform target distribution.
Our algorithm does not assume the ``true'' distribution of classes is uniform. Instead, we let a roboticist choose \textit{any} desired ``target'' distribution of classes that they want to collect in the cloud for their analytics or ML use case. Our theory is general -- once the target distribution is chosen and fixed, we have a convex loss function between the current and target distribution, which guarantees convergence via Theorems 1-3. 

Often, in practice, a roboticist might want an equal distribution of classes to train a robust ML model, which was the case we showed in Figure 3 of the main paper. However, if the roboticist wants to create a model to especially focus on weak points (safety critical examples) for which we have few examples in the cloud, they can select a non-uniform target distribution. Fig. \ref{fig:skewed_plots} shows that our algorithm is able to easily converge to a non-uniform target distribution. In the barplot, the initial data distribution among robots is highly skewed (grey). The target distribution is pink and is clearly non-uniform. The left panel illustrates our Interactive and Oracle policies quickly converge as our theory guarantees. Moreover, the barplot shows the final data distribution achieved in the cloud under the Interactive and Oracle policies closely matches the target (in pink) and is much closer than the heuristic uniform sampling and greedy sampling benchmarks (orange and red). Finally, we repeat this experiment for another non-uniform target distribution in the bottom panel and see it also converges, as expected by our theory. 

}

\clearpage

\subsection{Performance Gap Between \Oracle \space and \Greedy}
\label{sec:performance_gap} 
Here, we prove the performance gap between \Oracle \space and \Interactive. All definitions with \textit{feasible} mean that they satisfy the constraints in Eq. \ref{eq:oracleopt}, \ref{eq:indvopt}, and \ref{eq:interactiveopt}. 

\begin{definition}[Feasible space of all robots]
A feasible space of all robots under \Oracle \space is the Minkowski sum of all robots' feasible spaces (see Def. \ref{def:fspace}). The feasible space of all robots is:
$$\fspace{ }{r}=\{\sum_{i=1}^{\ndevice}\faction{i}{r}~|~\forall i=1,...,\ndevice, ~\faction{i}{r} \in \fspace{i}{r}\}.$$
\label{def:allfspace}
\end{definition}

\begin{definition}[Feasible actions]
We define the optimal feasible action for robot $i$ in round $r$ as $\faction{i}{*,r}$. The symbol of $*$ can denote g or o, standing for \Greedy \space or \Oracle \space respectively. Also, we define the optimal action $\action{*,r}{i}$ with the left inverse of $\dataobmatrix{i}{r}$ as ${\dataobmatrix{i}{r}}^{\dagger}$. The actions are obtained by solving the optimization problems under different scenarios like \Oracle~or \Greedy as follows: 

\begin{subequations}
\begin{align}
\notag
\faction{i}{g,r} = \arg\min_{\faction{i}{r}} \quad & \loss{}{}(\clouddataset{r}+\faction{i}{r},  \datasettarget).\\ \notag
\text { subject to: } \quad & \faction{i}{r} \in \fspace{i}{r} \notag
\end{align}
\end{subequations}

\begin{equation*}
\begin{aligned}
 \nonumber
\faction{i}{o,r}=
\arg\min_{\faction{i}{r}} \quad & \loss{}{}(\clouddataset{r} + \sum_{i=1}^{\ndevice}\faction{i}{r}, \datasettarget).\\
\text{subject to:} \quad & \faction{i}{r} \in \fspace{i}{r}, \forall i=1,...,\ndevice
\end{aligned}
\end{equation*}

$$\action{g,r}{i}={\dataobmatrix{i}{r}}^{\dagger}\faction{i}{g,r}.$$
$$\action{o,r}{i}={\dataobmatrix{i}{r}}^{\dagger}\faction{i}{o,r}.$$
\label{def:faction}
\end{definition}

Now, we compare the optimal values of \Oracle~and \Greedy~and show that \Oracle~outperforms \Greedy. Then we formulate the performance bound between \Oracle~and \Greedy~with a lower bound.

\begin{definition}[Optimal values of loss functions]
For simplicity, we define the optimal values of loss functions under \Oracle~and \Greedy~policies as $\loss{ }{g}$ and $\loss{ }{o}$. 
These are the values of the loss functions resulting from feasible actions: 
$$\loss{ }{g,r} = \loss{ }{ }(\clouddataset{r}+\sum_{i=1}^{\ndevice}\faction{i}{g,r}, \datasettarget),$$
$$\loss{ }{o,r}  =\loss{ }{ }(\clouddataset{r}+\sum_{i=1}^{\ndevice}\faction{i}{o,r}, \datasettarget).$$
\label{def:loss}
\end{definition}

\begin{theorem}[\Oracle \space outperforms \Greedy]
The optimal value of \Greedy~policy $\loss{ }{g,r}$ is always greater than or equal to the optimal value of \Oracle~policy $\loss{ }{o,r}$, i.e. $\loss{ }{g,r}\geq\loss{ }{o,r}$.
\end{theorem}
\begin{proof}
By Def. \ref{def:allfspace} and \ref{def:faction}, $\sum_{i=1}^{\ndevice}\faction{i}{g,r} \in \fspace{ }{r}$ and $\sum_{i=1}^{\ndevice}\faction{i}{o,r} \in \fspace{ }{r}$. 
By Def. \ref{def:loss}, $\loss{ }{o}$ is the minimum of the loss function under feasible space $\fspace{ }{r}$, so all other loss functions generated by vectors in the same feasible space must be larger. Therefore, $\loss{ }{g,r}\geq\loss{ }{o,r}$.
\end{proof}

\begin{theorem} [Performance gap of \Oracle \space and \Greedy]
We use Def. \ref{def:loss} and the triangle inequality to show the performance gap between \Oracle \space and \Greedy.
$$
\begin{aligned}
0\leq \loss{ }{g,r} - \loss{ }{o,r} =
 & \|\datasettarget-\clouddataset{r}-\sum_{i=1}^{\ndevice}\faction{i}{g,r}\| -  \|\datasettarget-\clouddataset{r}-\sum_{i=1}^{\ndevice}\faction{i}{o,r}\| \\ 
\leq  & \|\sum_{i=1}^{\ndevice} (\faction{i}{o,r}-\faction{i}{g,r}) \|
\end{aligned}
$$
\label{theorem bound}
\end{theorem}

As a special illustrative case, if all $\fspace{i}{r}$s are identical, then $\faction{i}{g,r}=\faction{i}{o,r}$. 
According to Thm. \ref{theorem bound}, $0\leq \loss{ }{g,r} - \loss{ }{o,r} \leq0$, hence $\loss{ }{g,r} - \loss{ }{o,r} =0$. \Greedy~is the optimal policy in this case, and there is no need to do cooperative data sharing. This could arise, for example, when all robots have the same vision model uncertainty and same local data distribution.

Next, we show an easy way to obtain the lower bound of \Oracle, using the Euclidean norm as an example. We create a new relaxation of Eq. \ref{eq:oracleopt} by removing the first constraint in Eq. \ref{eq:oracleopt}. That is, the number of the data-points uploaded need not be positive. In this case, since robots can upload \textit{negative} data-points, any combination of data-point is feasible as long as its sum is less than or equal to $\ncache$. Thus, confusion matrices of robots do not
matter here, and this mimics a case with no perceptual uncertainty. 

\begin{lemma}[Lower bound of \Oracle] \label{lemma:lowerbound}
The relaxation of Eq. \ref{eq:oracleopt} by removing the first constraint in Eq. \ref{eq:oracleopt} is the lower bound of \Oracle.
\end{lemma}
\begin{proof}
The original feasible set of the optimization problem is a subset of the new feasible set since we expand the set by removing a constraint from the original problem. 
Hence, we know the new optimal value is less than or equal to the original one.
Namely,

$$\begin{aligned}
\loss{ }{low,r}=\min & \quad \loss{}{} (\clouddataset{r}+\sum_{i=1}^{\ndevice}\faction{i}{r}, \datasettarget)\leq\loss{ }{o,r}.\\
\textit{subject to:} & \quad 1^T \cdot \faction{r}{i} \leq  N_{\text{cache}} ; \hspace{1mm} \forall \hspace{1mm} i = 1, \ldots, \ndevice
\end{aligned}$$

For $L_2$ norm, a closed-form solution of $\loss{ }{low,r}$ can be obtained by projecting the objective value to the feasible space:
$$\loss{ }{low,r}=\max(\mathbf{1}^\top(\datasettarget-\clouddataset{r})-\ncache\times\ndevice, 0)\times\sqrt{\nclass}.$$
\end{proof}

\subsection{Theorem \ref{theorem:converge_eventually}: While loop in Alg. \ref{alg:train} converges eventually} 
\label{sec:converge_eventually} 
We first show that there is a unique solution of \Interactive \space then show that Alg. \ref{alg:train} will converge to that solution. 

\begin{lemma}[Uniqueness of \Interactive \space solution]
The optimal solution of \Interactive \space $\sum_{i=1}^{\ndevice}\faction{i}{int,r}$ is unique.
\label{lemma:interactiveunique}
\end{lemma}
\begin{proof}
We use proof by contradiction. First, we know $\sum_{i=1}^{\ndevice}\faction{i}{int,r} \in \fspace{ }{r}$, and $\fspace{ }{r}$ is a convex set.
If there exist more than two optimal solutions, we arbitrarily pick two of them and name them $\faction{ }{int,r}$ and $\faction{ }{'int,r}$.
Since $\loss{ }{ }(\cdot,\cdot)$ is strictly convex,

$$
\begin{aligned}
\loss{ }{ }(\clouddataset{r}+\frac{1}{2}[\faction{ }{int,r}+\faction{ }{'int,r}], \datasettarget)
& < \frac{1}{2}[\loss{}{}(\clouddataset{r}+\faction{ }{int,r}, \datasettarget)+ \loss{ }{ }(\clouddataset{r}+\faction{ }{'int,r},\datasettarget)]\\
& =\loss{ }{ }(\clouddataset{r}+\faction{ }{int,r}, \datasettarget).
\end{aligned}
$$
Then $\frac{1}{2}[\faction{ }{int,r}+\faction{ }{'int,r}]$ achieves a lower loss function and contradicts with our assumption that $\faction{ }{int,r}$ and $\faction{ }{'int,r}$ are optimal solutions. Hence, $\sum_{i=1}^{\ndevice}\faction{i}{int,r}$ is unique.
\end{proof}

\begin{theorem-non}[Convergence Eventually]
The \textit{while} loop (lines \ref{ln:while_loop} - \ref{ln:while_loop_end}) in Alg. \ref{alg:train} will eventually converge.  
\end{theorem-non}
\begin{proof}
For the proof of convergence, refer to Theorem 2 of \cite{AnalysisBRPotentialGames}. 
The potential function defined in \cite{AnalysisBRPotentialGames} corresponds to the negative value of our objective function, as stated in section \ref{app:whypotential}. 
The random revision law there is replaced by our deterministic order of updates in line \ref{ln:for2_loop}.
% In Eq. \ref{eq:interactiveopt}, all robots have a common striclty convex objective funciton.
% Thus, all robots can only striclty decrease the objective funciton. 
From Lemma \ref{lemma:interactiveunique}, we know there is only one unique solution, thus eventually Alg. \ref{alg:train} will converge to it. 
\end{proof}

Intuitively, a potential game with a strictly concave potential function will converge eventually since all players (robots in our case) strictly increase the potential function.

\subsection{Theorem \ref{theorem:orable_inter_same}: \Interactive \space converges to \Oracle} \label{sec:oracle_inter_same_sec}

We prove our proposed method \Interactive \space described in Eq. \ref{eq:interactiveopt} and Alg. \ref{alg:train} is equivalent to \Oracle \space as described in Thm. \ref{theorem:orable_inter_same}. 
We discuss two cases respectively: 
$\mathbf{1}^\top(\datasettarget-\clouddataset{r}) > \ndevice \times \ncache$ and 
$\mathbf{1}^\top(\datasettarget-\clouddataset{r}) \leq \ndevice \times \ncache$.
The first case holds for all rounds except for the last round that reaches the target distribution, upon which data collection terminates (see Thm. \ref{theorem:converge_one}). 
When $\mathbf{1}^\top(\datasettarget-\clouddataset{r}) > \ndevice\times\ncache$ holds, \Interactive \space  will certainly converge to \Oracle \space in one while loop execution (running Alg. \ref{alg:train} line \ref{ln:while_loop} - \ref{ln:while_loop_end} once). While in the last round, $\mathbf{1}^\top(\datasettarget-\clouddataset{r}) \leq \ndevice \times \ncache$ holds, and it takes more than one execution to converge.

\begin{lemma}[Uniqueness of \Oracle \space solution]
The optimal feasible action of \Oracle, namely $\sum_{i=1}^{\ndevice}\faction{i}{o,r}$, is unique.
\label{lemma:unique}
\end{lemma}
\begin{proof}
The proof is similar to Lemma \ref{lemma:interactiveunique}.
We use proof by contradiction. First, we know $\sum_{i=1}^{\ndevice}\faction{i}{o,r} \in \fspace{ }{r}$, and $\fspace{ }{r}$ is a convex set.
If there exist more than two optimal solutions, we arbitrarily pick two of them and name them $\faction{ }{ }$ and $\faction{ }{'o,r}$.
Since the loss $\loss{ }{ }(\cdot,\cdot)$ is a strictly convex function,

$$
\begin{aligned}
\loss{ }{ }(\clouddataset{r}+\frac{1}{2}[\faction{ }{o,r}+\faction{ }{'o,r}], \datasettarget)
& < \frac{1}{2}[\loss{ }{ }(\clouddataset{r}+\faction{ }{o,r}, \datasettarget)+ \loss{ }{ }(\clouddataset{r}+\faction{ }{'o,r}, \datasettarget)]\\
& =\loss{ }{ }(\clouddataset{r}+\faction{ }{o,r}, \datasettarget).
\end{aligned}
$$
Then $\frac{1}{2}[\faction{ }{o,r}+\faction{ }{'o,r}]$ achieves a lower loss function and contradicts with our assumption that $\faction{ }{o,r}$ and $\faction{ }{'o,r}$ are optimal solutions. Hence, $\sum_{i=1}^{\ndevice}\faction{i}{o,r}$ is unique.
\end{proof}

\begin{theorem-non}[\Interactive \space converges to \Oracle]
The while loop in Alg. \ref{alg:train} line \ref{ln:while_loop} - \ref{ln:while_loop_end} is guaranteed to return action $\action{int,r}{i}$ that is equal to the \Oracle \space policy's action, $\action{o,r}{i}$. 
$\action{int,r}{i}$ denotes the action of robot $i$ at the end of round $r$ using the \Interactive \space policy. Similarly, $ \action{o,r}{i}$ denotes the action of \Oracle \space policy.
\end{theorem-non}

\begin{proof}
The convergence (optimality) conditions for the convex optimization problems of all robots are of this form with the gradient of the loss function $\nabla_{\faction{i}{int,r*}}\|\datasettarget-\clouddataset{r}-\sum_{j=1}^{\ndevice}\faction{i}{int,r*}\|$:
$$
\begin{aligned}
&\forall i, \faction{i}{ }\in\fspace{i}{r},\\
&(\faction{i}{ }-\faction{i}{int,r*})^\top\nabla_{\faction{i}{int,r*}}\|\datasettarget-\clouddataset{r}-\faction{i}{int,r*}-\sum_{i\not=j,j=1}^{\ndevice}\faction{j}{int,r*}\|\\
=&(\faction{i}{ }-\faction{i}{int,r*})^\top\nabla_{\faction{i}{int,r*}}\|\datasettarget-\clouddataset{r}-\sum_{j=1}^{\ndevice}\faction{i}{int,r*}\|\geq 0.
\end{aligned}
$$
By the chain rule, 
$$
\begin{aligned}
&\nabla_{\faction{i}{int,r*}}\|\datasettarget-\clouddataset{r} - \sum_{j=1}^{\ndevice}\faction{j}{int,r*}\|=
\nabla_{\sum_{j=1}^{\ndevice}\faction{j}{int,r*}}\|\datasettarget-\clouddataset{r} - \sum_{j=1}^{\ndevice}\faction{j}{int,r*}\|.
\end{aligned}
$$
Thus, summing up the optimality conditions of all robots, we get:
$$
\begin{aligned}
&\sum_{i=1}^{\ndevice}
(\faction{i}{ }-\faction{i}{int,r*})^\top\nabla_{\faction{i}{int,r*}}\|\datasettarget-\clouddataset{r}-\sum_{j=1}^{\ndevice}\faction{j}{int,r*}\|\\
=&(\sum_{i=1}^{\ndevice}\faction{i}{ }-\sum_{i=1}^{\ndevice}\faction{i}{int,r*})^\top
\nabla_{\sum_{j=1}^{\ndevice}\faction{i}{int,r*}}\|\datasettarget-\clouddataset{r}-\sum_{j=1}^{\ndevice}\faction{i}{int,r*}\|\geq 0.
\end{aligned}
$$
This implies the optimality condition of \Oracle \space is:
$$
\begin{aligned}
&\forall \sum_{i=1}^{\ndevice}\faction{i}{ }\in\fspace{}{r},\;\;
(\sum_{i=1}^{\ndevice}\faction{i}{ }-\sum_{i=1}^{\ndevice}\faction{i}{o,r})^\top
\nabla_{\sum_{j=1}^{\ndevice}\faction{i}{o,r}}\|\datasettarget-\clouddataset{r}-\sum_{j=1}^{\ndevice}\faction{i}{o,r}\|\geq 0.
\end{aligned}
$$

% The while loop explores the sum of feasible spaces of all robots iteratively. By definition \ref{def:fspace} and \ref{def:allfspace}, it is exactly searching in $\fspace{ }{r}$.
% Every robot calculates action $\action r i$ and improves $\loss{ }{ }(\mathcal{D_{\text{target}}},\clouddataset{r}+\sum_{i=1}^{\ndevice}{\dataobmatrix{i}{r}}^{\dagger}\action{o,r}{ })$. Since the feasible spaces of \Oracle and \Interactive are identical and 
We know there is only one unique solution of \Oracle \space  from Lemma \ref{lemma:unique}, so \Interactive \space  will converge to \Oracle \space in Alg. \ref{alg:train} line \ref{ln:while_loop} - \ref{ln:while_loop_end}, and $$\action{int,r}{i} = \action{o,r}{i} = {\dataobmatrix{i}{r}}^{\dagger} \faction{i}{o,r}.$$
\end{proof}

\begin{lemma}[Sum of feasible actions lies on a hyperplane]
For \Oracle \space and \Greedy, the sum of action lies on the same hyperplane $\mathbf{1}^\top v=\ndevice \times \ncache$ when $\mathbf{1}^\top(\datasettarget-\clouddataset{r}) > \ndevice \times \ncache$.
\label{lemma:hyperplane}
\end{lemma}
\begin{proof}
Since $\datasettarget$ lies outside $\fspace{ }{r}$, the closest point to it must lie on the boundary of the convex set. Thus, $\sum_{i=1}^{\ndevice} \faction{i}{o,r}$ lies at the edge of $\fspace{ }{r}$, the hyperplane $\mathbf{1}^\top v=\ndevice \times \ncache$. Thus, 
% And all columns in feasible data matrix are probability vectors, so 
$$\mathbf{1}^\top \sum_{i=1}^{\ndevice} \faction{i}{o,r}
%=\sum_{i=1}^{\ndevice} \mathbf{1}^\top {\dataobmatrix{i}{r}}\action{o,r}{i}
=\ndevice \times \ncache.$$

Every shared action in Alg. \ref{alg:train} line \ref{ln:share_init_actions} is the \Greedy \space action $\action{g,r}{i}$ and the corresponding feasible action $\faction{i}{g,r}$ lies at the edge of $\fspace{i}{r}$, the hyperplane $\mathbf{1}^\top v= \ncache$ for the same reason as above. 
Thus, we know: 
$$\mathbf{1}^\top\sum_{i=1}^{\ndevice} \faction{i}{g,r}
=\sum_{i=1}^{\ndevice} \mathbf{1}^\top \faction{i}{g,r}
%=\sum_{i=1}^{\ndevice} \mathbf{1}^\top \action{g,r}{i}
=\ndevice\times\ncache$$
The sum of greedy feasible actions also lies on the same hyperplane $\mathbf{1}^\top v= \ndevice \times \ncache$.
\end{proof}

Now, using the fact that the sum of feasible actions lies on the same hyperplane from Lemma \ref{lemma:hyperplane}, we can show that the while loop in Alg. \ref{alg:train} line \ref{ln:while_loop} - \ref{ln:while_loop_end} will terminate in one iteration.

\subsection{Theorem \ref{theorem:converge_one}: While loop converges in one iteration} \label{sec:one_iteration} 

\begin{theorem-non}[Convergence in one iteration] 
For cases when the total number of uploadable data-points is less than the difference between target cloud dataset $\datasettarget$ and current cloud dataset $\clouddataset{r}$, namely
$\mathbf{1}^\top(\datasettarget-\clouddataset{r}) > \ndevice \times \ncache$, the while loop in Alg. \ref{alg:train} line \ref{ln:while_loop} - \ref{ln:while_loop_end} will terminate in one iteration.
\end{theorem-non}
\begin{proof}
Since the optimal solution of \Oracle \space is unique from Lemma \ref{lemma:unique}, we know the update direction of solution (the vector from the previous solution pointing to  the new solution) in the first optimization execution in line \ref{ln:while_loop} - \ref{ln:while_loop_end} is the vector pointing from the \Greedy\space feasible solution $\sum_{i=1}^{\ndevice}\faction{i}{g,r}$ to the \Oracle \space  solution $\sum_{i=1}^{\ndevice}\faction{i}{o,r}$. Both points lie on the hyperplane
    $\mathbf{1}^\top v=\ndevice\times\ncache$ by Lemma \ref{lemma:hyperplane}. 
Also, all feasible spaces in Eq. \ref{eq:interactiveopt} intersect with the hyperplane $\mathbf{1}^\top v=\ndevice\times\ncache$, so all update directions in line \ref{ln:while_loop} - \ref{ln:while_loop_end} during the while loop lie on the same hyperplane until the solutions converge. 
% Last, we interpret the optimization in line \ref{ln:while_loop} - \ref{ln:while_loop_end} as searching for the closest point to $\sum_{i=1}^{\ndevice}\faction{i}{o,r}$ in its feasible space $\fspace{i}{r}$.

Let the solution after the first iteration of the while loop be $v_r^\text{iter}$ and the solutions of each for loop execution before it be $$v_r^{\text{for},j}, \text{for } j=1,...,\ndevice.$$
Note that, 
$$\sum_{i=1}^{\ndevice} \faction{i}{o,r} \in \{\faction{ }{ }:\mathbf{1}^\top v=\ndevice\times\ncache\},$$
$$ v_\text{iter} \in \{\faction{ }{ }:\mathbf{1}^\top v=\ndevice\times\ncache\},$$
$$v_r^{\text{for},j} \in \{\faction{ }{ }:\mathbf{1}^\top v=\ndevice\times\ncache\}, \text{for } j=1,...,\ndevice,$$
since all the updates happen on the hyperplane.

We then assume $v_\text{iter}$ is not the solution of \Oracle \space, $\sum_{i=1}^{\ndevice} \faction{i}{o,r}$, and prove it is wrong by contradiction. 
If they are not identical, let the difference between solutions of \Oracle \space and the first iteration be $$\Delta v= \sum_{i=1}^{\ndevice} \faction{i}{o,r}- v_\text{iter}\neq0.$$

    $\Delta v$ is the same direction as all update directions in line \ref{ln:while_loop} - \ref{ln:while_loop_end}. All $v_r^{\text{for},j} +\alpha\Delta v$ are infeasible ($\not\in \fspace{j}{r}$) for any $j$ and an arbitrary small step size of update $\alpha>0$ because all $v_r^{\text{for},j}$ are already optimal solutions that cannot move further in the update directions. 
Hence, $$\sum_{i=1}^{\ndevice}\faction{i}{o,r}= v_\text{iter} + \Delta v \not \in \fspace{ }{r}.$$
This contradicts with Def. \ref{def:faction}, so we prove that $\Delta v = 0$ and $\sum_{i=1}^{\ndevice} \faction{i}{o,r} = v_\text{iter}$.
In other words, the while loop in Alg. \ref{alg:train} line \ref{ln:while_loop} - \ref{ln:while_loop_end} will terminate in one iteration.
\end{proof}

\subsection{Proposition \ref{prop:total_msgs}: The total number of messages passed between the robots.} 
\label{sec:total_num_messages} 

We first calculate the number of messages passed in every iteration of Alg. \ref{alg:train} using an \textit{un-optimized} method of communication that requires $O(\ndevice^2)$ messages. Then, we show a simple, optimized method that requires only $O(\ndevice)$ messages per loop. 

\begin{proposition}[Total Number of Messages] 
The total number of messages passed between the robots in line \ref{ln:share_init_actions} will be $\ndevice^2-\ndevice$. While in each iteration (for loop line \ref{ln:for2_loop} - \ref{ln:for2_loop_end}), the number is also $\ndevice^2-\ndevice$.
\label{prop:total_msgs}
\end{proposition}

\begin{proof}
Each robot $i$ shares its decision $\dataobmatrix{i}{r} \action{r}{i}$ with $(\ndevice-1)$ other robots, and this process repeats $\ndevice$ times for all robots. Hence, the total numbers of messages passed between the robots in line \ref{ln:share_init_actions} and for loop line \ref{ln:for2_loop} - \ref{ln:for2_loop_end} are both $$\ndevice\times(\ndevice-1)=\ndevice^2-\ndevice.$$
\end{proof}

\subsubsection{An Optimized Method with only $O(\ndevice)$ messages}
\label{sec:opt_comm}

Our key insight to reduce communication, shown in Fig. \ref{fig:message_pass}, is that robots only need to share their individual actions initially and afterwards can only share \textit{sums} of their actions with each other.

\begin{figure*}[t]
    \includegraphics[width=1.0\columnwidth]{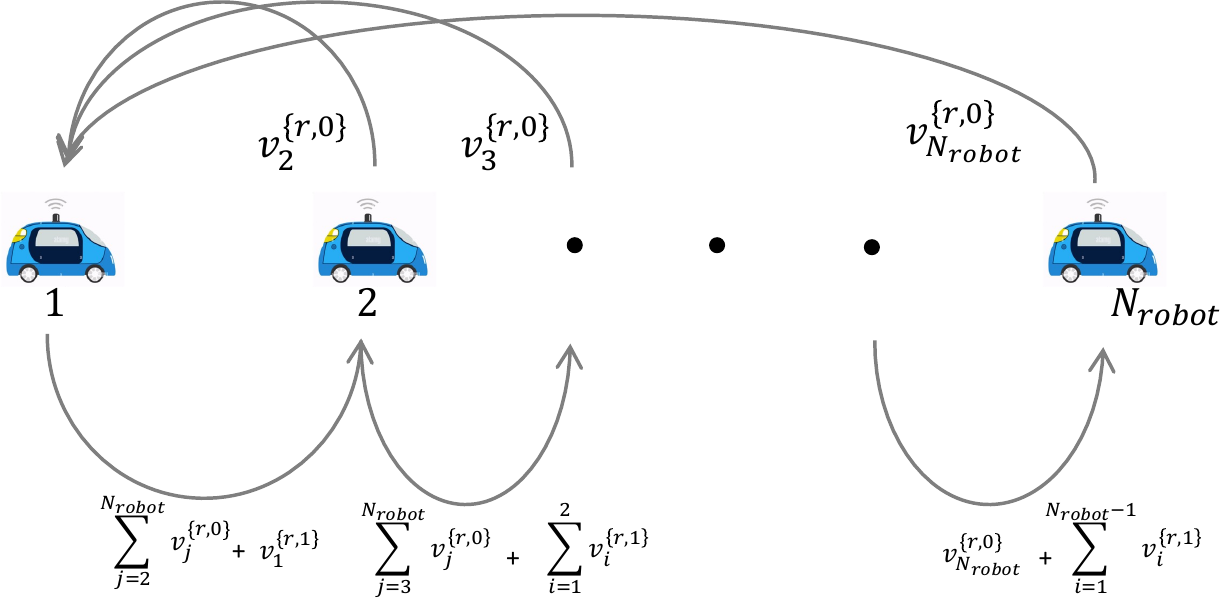}
    \vspace{1em}
    \caption{\small{\textbf{Communication Optimization in Alg. 1 While loop: } First, each robot shares it greedy actions (grey arrows facing left) $v^{r,0}_i$. Then, each robot passes the \textit{sum} of optimized actions $v^{r,1}_i$ and other robots' actions $v^{r,0}_j$ as opposed to \textit{individual} actions, leading to $O(\ndevice)$ messages.}} 
     \label{fig:message_pass}
\end{figure*}

As shown in Fig. \ref{fig:message_pass}, let us denote a feasible action on round $r$  by $v^r_i = P^r_ia^r_i$. Further, let us index iterations of communication \textit{within} a loop by $k$, meaning $v^{r,0}_i$ is the \textit{initial greedy} action from robot $i$ at round $r$ (i.e., at iteration $0$). After solving Prob. 3 once and multiplying by $P^r_i$, the next action is given by $v^{r,1}_i$. 
As shown in Fig. \ref{fig:message_pass}, all robots send their initial greedy action $v^{r,0}_j$ to robot $1$ for $j=2\ldots\ndevice$. This amounts to $\ndevice-1$ messages sent. Then, robot $1$ solves Prob. 3, assuming all other robots' actions are fixed, to generate $v^{r,1}_1$. The sum of the new optimized action $v^{r,1}_1$ and previous unoptimized actions $\sum_{j=2}^{\ndevice} v^{r,0}_j$ is sent to robot $2$. Robot $2$ then subtracts its current greedy action $v^{r,0}_2$ in Eq.
\ref{eq:dataset_inter} and solves Prob. 3 again. The process repeats until we reach robot $\ndevice$, leading to another $\ndevice-1$ messages. As such, for each while loop iteration, we only need $(\ndevice-1) + (\ndevice-1) = 2(\ndevice-1)$ messages, so $O(\ndevice)$ messages as opposed to $O(\ndevice^2)$.

\end{document}